\documentclass[romanappendices]{IEEEtran}

\usepackage{amsmath,amsfonts,amssymb,amsthm,bbm,accents}
\usepackage{graphicx,enumerate,url}
\usepackage{setspace}
\usepackage{microtype}
\usepackage{hyperref}

\usepackage{algorithm}
\usepackage{algpseudocode}

\usepackage{enumitem}
\usepackage{cite}

\usepackage{subfigure}
\usepackage{tikz}
\usetikzlibrary{arrows,arrows.meta}
\usepackage{pgfplots}
\usepgfplotslibrary{groupplots}
\usepgfplotslibrary{fillbetween}
\usetikzlibrary{external}
\tikzsetexternalprefix{externalize/}
\tikzset{external/optimize=false}

\usepackage{color}
\definecolor{mygreen}{rgb}{0.10,0.50,0.10}

\newtheorem{theorem}{Theorem}
\newtheorem{proposition}{Proposition}

\newtheorem{lemma}{Lemma}
\newtheorem{assumption}{Assumption}
\newtheorem{definition}{Definition}
\theoremstyle{definition}
\newtheorem{remark}{Remark}
\newtheorem{example}{Example}

\DeclareMathOperator{\E}{\mathbb{E}}

\DeclareMathOperator*{\argmin}{argmin}
\DeclareMathOperator*{\argmax}{argmax}

\DeclareMathOperator{\sto}{s.t.}

\DeclareMathOperator{\indicator}{\mathbb{I}}

\newcommand{\multiset}[2]{%
\mathchoice{\left(\kern-0.5em{\binom{#1}{#2}}\kern-0.5em\right)}
           {\bigl(\kern-0.3em{\binom{#1}{#2}}\kern-0.3em\bigr)}
           {\bigl(\kern-0.3em{\binom{#1}{#2}}\kern-0.3em\bigr)}
           {\bigl(\kern-0.3em{\binom{#1}{#2}}\kern-0.3em\bigr)}}

\newcommand{\lam}{\ensuremath{\lambda}}

\newcommand{\calA}{\ensuremath{\mathcal{A}}}
\newcommand{\calB}{\ensuremath{\mathcal{B}}}

\newcommand{\calD}{\ensuremath{\mathcal{D}}}

\newcommand{\calF}{\ensuremath{\mathcal{F}}}

\newcommand{\calK}{\ensuremath{\mathcal{K}}}
\newcommand{\calL}{\ensuremath{\mathcal{L}}}

\newcommand{\calS}{\ensuremath{\mathcal{S}}}

\newcommand{\setR}{\ensuremath{\mathbb{R}}}

\def\st/{\textsuperscript{st}}
\def\nd/{\textsuperscript{nd}}
\def\rd/{\textsuperscript{rd}}
\def\th/{\textsuperscript{th}}

\def\nnil{\nil}
\newcounter{prob}

\author{Miguel Calvo-Fullana, Santiago Paternain, Luiz F. O. Chamon and Alejandro Ribeiro
\thanks{M. Calvo-Fullana is with the Dept. of Information and Communication Technologies, Universitat Pompeu Fabra, Barcelona, Spain (email: miguel.calvo@upf.edu).}%
\thanks{S. Paternain is with the Dept. of Electrical Computer and Systems Engineering, Rensselaer Polytechnic Institute, Troy, NY, USA (email: paters@rpi.edu).}%
\thanks{L. F. O. Chamon is with the Excellence Cluster for Simulation Technology, University of Stuttgart, Stuttgart, Germany (email: luiz.chamon@simtech.uni-stuttgart.de).}%
\thanks{A. Ribeiro is with the Dept. of Electrical and Systems Engineering, University of Pennsylvania, Philadelphia, PA, USA (email: aribeiro@seas.upenn.edu).}%
\thanks{This work was supported in part by ARL DCIST CRA W911NF-17-2-0181, the Deutsche Forschungsgemeinschaft (DFG, German Research Foundation) under Germany's Excellence Strategy (EXC 2075-390740016), and Spain's Agencia Estatal de Investigaci\'on under grant RYC2021-033549-I.}
}

\title{State Augmented Constrained Reinforcement Learning: Overcoming the Limitations of Learning with Rewards}

\begin{document}

\maketitle


\begin{abstract}
A common formulation of constrained reinforcement learning involves multiple rewards that must individually accumulate to given thresholds. In this class of problems, we show a simple example in which the desired optimal policy cannot be induced by any weighted linear combination of rewards. Hence, there exist constrained reinforcement learning problems for which neither regularized nor classical primal-dual methods yield optimal policies. This work addresses this shortcoming by augmenting the state with Lagrange multipliers and reinterpreting primal-dual methods as the portion of the dynamics that drives the multipliers evolution. This approach provides a systematic state augmentation procedure that is guaranteed to solve reinforcement learning problems with constraints. Thus, as we illustrate by an example, while previous methods can fail at finding optimal policies, running the dual dynamics while executing the augmented policy yields an algorithm that provably samples actions from the optimal policy.
\end{abstract}

\section{Introduction}

Complex behavior is a fundamental trait of autonomous agents that often arises as a response to conflicting requirements. When learning these behaviors using a reinforcement learning~(RL) approach, the agent receives a reward signaling the level of attainment of each of the task requirements~\cite{sutton2018reinforcement}. Typically, the ensuing rewards are weighted and combined resulting in a regularized multi-objective reinforcement learning problem~\cite{mannor2004geometric,di2012policy,tamar2013variance,peng2018deepmimic}. While popular and sometimes effective, this approach suffers from several drawbacks, not the least of which is the fact that combination coefficients are problem-specific and must be manually selected. A time-consuming calibration process, which causes severe computational overhead~\cite{mania2018simple,peng2018deepmimic}. 

Alternatively, requirements can be made explicit in a constrained reinforcement learning (CRL) problem \cite{paternain2019constrained}. This is not altogether unrelated to regularized RL approaches because algorithms to solve CRL rely on Lagrangian dual formulations and Lagrangians \emph{are} weighted combinations of rewards (Section \ref{S:constrainedRL}). The difference is that dual methods incorporate rules to adjust the value of the multipliers, thereby eliminating the design overhead of their manual selection \cite{bhatnagar2012online, chow2015risk, tessler2018reward, paternain2022safe,ding2020natural}. 

\begin{figure}[t]
    \centering
	\includegraphics[scale=0.75]{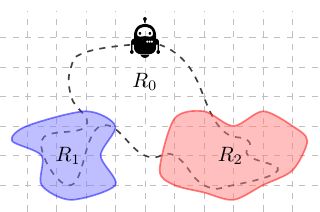} 					
	\includegraphics[scale=0.75]{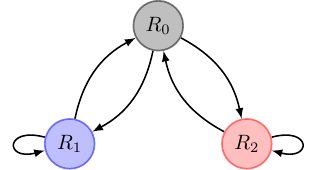} 					
	\caption{\emph{Monitoring problem}. An agent must spend $c_1$ of its time at $R_1$ and $c_2$ of its time at $R_2$. Continuous state representation on the left and discrete state abstraction on the right. Solving this problem requires a state augmentation method we develop here. }
    \label{fig:example}
\end{figure}

However related, CRL and regularized RL formulations are not equivalent. In fact, the motivation for this paper is the following observation:
\begin{description}
\item[(O1)] CRL can solve problems regularized RL cannot solve.
\end{description}
The precise meaning of this statement is that there is a class of problems for which no regularization parameters enable learning a policy that satisfies the problem requirements. There is, moreover, nothing contrived about this class. A case in point is the \emph{monitoring problem} illustrated in Figure~\ref{fig:example}. The agent is tasked with monitoring different regions of the environment; it collects unit rewards whenever it visits a target region and it must spend a certain fraction of its time in each of them. No fixed combination of these rewards can ensure that the agent learns a policy that satisfies the monitoring requirements~(Example~\ref{ex_monitoring}).

This observation foreshadows a central role for CRL in autonomy. Even simple tasks cannot be specified by formulations without constraints. It also casts some doubt on dual formulations given that Lagrangians are weighted combinations of rewards. This worry turns out to be well founded since the following is also true:
\begin{description}
\item[(O2)] Dual methods do not, in general, converge to optimal CRL policies.
\end{description}
The monitoring problem of Figure~\ref{fig:example} is an example of a  problem for which dual methods fail (Sections \ref{sec_policy_switching} and \ref{sec_primal_dual}). Indeed, as dual methods rely on the solution of regularized (weighted) problems, the inability to obtain solutions to such problems poses a fundamental obstacle to solving CRL problems using dual descent methods (Section \ref{sec_primal_recovery}).

Observation (O1) highlights the fundamental importance of CRL and Observation (O2) the lack of an algorithm that can provably solve CRL. The main contribution of this paper is the following:
\begin{description}
\item[(C1)] We develop a state augmentation procedure that enables us to solve CRL (Section \ref{sec_algorithm}).
\end{description}
The proposed state augmentation is based on reinterpreting Lagrangian weights as part of the state space of the Markov decision process (MDP). As such, this procedure is systematic and independent of the problem instance. Contribution (C1) hinges upon the following technical contribution:
\begin{description}
\item[(C2)] We prove that dual gradient descent dynamics generate \emph{trajectories} that are almost surely feasible and have expected costs that are near-optimal (Theorem \ref{theo_main_intro}). 
\end{description}
Contribution (C2) is the technical result that makes contribution (C1) possible. It claims convergence of \emph{trajectories} that are generated by dual gradient descent dynamics, but it does not claim convergence to an optimal policy. The latter can be shown impossible through counterexamples~[cf.~(O2)]. The fact that trajectories result in feasibility and near-optimality without converging to optimal policies is the reason why state augmentation is necessary to solve CRL~[cf.~(C1)]. This peculiar situation is a unique aspect of CRL that is due to the unusual combination of a problem in which (i) there is no duality gap despite presence of constraints and objectives that are non-convex (Section \ref{sec_relationship}) and (ii) recovery of optimal primal variables from optimal dual variables is not always possible despite absence of duality gap (Section \ref{sec_primal_recovery}).

\subsection{Related Work}

MDPs are stochastic control processes which can be commonly found in the study of control problems \cite{shreve1978alternative}, with diverse applications ranging from highway management \cite{golabi1982statewide} to smart grids \cite{koutsopoulos2011control}, finance \cite{krokhmal2002portfolio,di2012policy} and robotics \cite{chow2015risk,gu2017deep,achiam2017constrained}, among many others. For these problems, when the model for the underlying MDP is known, optimal policies (controllers) can be found by way of dynamic programming techniques \cite{bertsekas1996neuro}. However, if the dynamics of the MDP are unknown, one must rely on sampled trajectories of the system in order to learn optimal policies, commonly through a methodology known as reinforcement learning \cite{sutton2018reinforcement}. 

Of additional importance is the fact that for many autonomous control tasks, agents are further required to satisfy certain task requirements to specified levels. Formally, these requirements are translated into constraints, resulting in constrained Markov decision processes (CMDP) \cite{altman1999constrained}. For explicit state and action constraints, under the assumption of model knowledge, these problems can be handled by classical control-theoretic techniques such as model predictive control (MPC) \cite{mayne2000constrained}. However, if there is no a priori knowledge on the system model, a RL methodology needs to be used, wherein system trajectories are sampled to learn an optimal and feasible policy. 

To address CMDPs in a RL framework, the most common are regularized approaches \cite{di2012policy,peng2018deepmimic,tamar2013variance}, which attempt to approximate the constrained reinforcement learning problem by the maximization of an unconstrained Lagrangian with specific weighted values \cite{borkar2005actor}. While practical, the choice of these weights is complicated and can result in a computationally exhausting process of hyper-parameter tuning that is often domain dependent. 
Other approaches, such as risk-sensitive formulations \cite{borkar2002risk} or ratio-based optimization \cite{suttle2021reinforcement} avoid the regularization process by introducing proxy objectives to address the inherent trade-offs of constrained problems, however guaranteeing explicit requirements via these formulations is not trivial. Alternatively, one can attempt to dynamically adjust the regularization weights themselves. This is the case of primal-dual approaches \cite{bhatnagar2012online,chow2015risk,tessler2018reward,paternain2022safe,ding2020natural}. These methods are based on the use of Lagrange multipliers as weights. Different primal steps can be used, with policy gradient \cite{tessler2018reward}, natural policy gradient \cite{ding2020natural} or actor-critic methods \cite{bhatnagar2012online} having being studied. However, these methods share, in essence, the same critical issue that plagues regularized approaches. Namely, convergence guarantees can be provided on value functions, but they do not result in an optimal policy capable of generating such guarantees. While primal averaging techniques \cite{anstreicher2009two,nedic2009approximate} could be potentially used to recover an optimal policy, they present their own set of limitations and restrictions (see Section \ref{sec_policy_switching} for a detailed discussion). To recover an optimal policy and guarantees is the main issue that we address in this work.

\section{Constrained Reinforcement Learning}
\label{S:constrainedRL}

Let~$t \in \mathbb{N} \cup \{0\}$ denote a time index and~$\calS \subset \setR^n$ and~$\calA \subset \setR^d$ be compact sets denoting the possible states and actions of an agent described by an MDP with transition probability $p\left(s_{t+1} \mid s_{t},a_t\right)$. The agent chooses actions sequentially based on a policy $\pi \in \mathcal{P}(\mathcal{S})$, where $\mathcal{P}(\mathcal{S})$ is the space of probability measures on $(\mathcal{A}, \mathcal{B}(\mathcal{A}))$ parametrized by elements of $\mathcal{S}$, where $\mathcal{B}(\mathcal{A})$ are the Borel sets of $\mathcal{A}$. In a constrained MDP the hypothesis is that the action taken by the agent at each state results in the collection of several rewards defined by the functions $r_i: \calS \times \calA \to \mathbb{R}$ for~$i = 0,\ldots, m$. Thus, when the agent is in state $s_t$ and takes action $a_t$, it collects rewards $r_i(s_t,a_t)$. As in any MDP, the interest is on the rewards the agent accumulates over time which are represented by the value functions
\begin{align}\label{eqn_value_function}
   V_i(\pi) \triangleq \lim_{T\to\infty}  \frac{1}{T} \E_{s,a\sim\pi}  \left[ \sum_{t = 0}^T r_i(s_t,a_t) \right] .
\end{align}
As customary \cite{sutton2018reinforcement}, we assume the existence of this limit. The necessary conditions can be made explicit in the discrete case, see \cite[Chapter 8]{puterman1994markov}. Typically, the value functions $V_i(\pi)$ are in conflict with each other. Policies $\pi$ that are optimal for some $V_i$ are not good for some other $V_j$. We therefore interpret the first of these rewards, $i=0$ as an objective to maximize and the remaining rewards as \emph{requirements} to satisfy. We then introduce the constrained optimization problem
\begin{subequations}
\begin{align*}
\tag{P-CRL}\label{P:constrainedRL}
   \max_{\pi \in \mathcal{P}(\mathcal{S})} \enskip   & \lim_{T\to\infty} \frac{1}{T}\E_{s,a\sim\pi} \left[ \sum_{t = 0}^T r_0(s_t,a_t) \right]    \\
   \sto 
		\enskip   & \lim_{T\to\infty} \frac{1}{T} \E_{s,a\sim\pi} \left[ \sum_{t = 0}^T r_i(s_t,a_t) \right]  \geq c_i \text{,} \, i = 1, \ldots, m\text{.}
\end{align*}
\end{subequations}
where the constant $c_i \in \setR$ represents the~$i$-th reward \emph{specification}. We denote by $ \Pi^\star $ the optimal solutions of \eqref{P:constrainedRL}. It is important for forthcoming discussions that $\Pi^\star$ is a set. There may be several policies $\pi^\star \in \Pi^\star$ that satisfy the constraints and attain the optimal cost $P^\star \triangleq V_0(\pi^\star)$. Ideally, our goal would be to find at least one element of this set.

A customary approach to solving \eqref{P:constrainedRL} is to consider a penalized version. Formally, introduce penalty coefficients $\lam_i$ associated with each of the constraints in \eqref{P:constrainedRL}
and define the Lagrangian $\calL(\pi,\lam)$ as
\begin{align}\label{eqn_lagrangian}
   \calL(\pi,\lambda) & \triangleq V_0(\pi) + \sum_{i=1}^m \lambda_i \Big(V_i(\pi) - c_i\Big)  .
\end{align}
In lieu of the constrained optimization \eqref{P:constrainedRL}, we settle for policies $\Pi(\lam)$ that maximize the Lagrangian in \eqref{eqn_lagrangian},
\begin{align}\label{eqn_lagrangian_maximizers}
   \Pi(\lambda) \triangleq \argmax_{\pi} V_0(\pi)  + \sum_{i=1}^m \lambda_i \Big(V_i(\pi) - c_i\Big).
\end{align}
The advantage of replacing \eqref{P:constrainedRL} with its Lagrangian form in \eqref{eqn_lagrangian_maximizers} is that the latter is a plain, unconstrained MDP which we can solve with standard reinforcement learning algorithms. Indeed, define the weighted reward
\begin{align}\label{eqn_lagrangian_reward}
   r_\lambda(s_t,a_t) = r_0(s_t,a_t) + \sum_{i=1}^m \lam_i \Big( r_i(s_t,a_t) - c_i \Big).
\end{align}
Given this definition we can reorder terms in \eqref{eqn_lagrangian_maximizers} to write $\Pi(\lam)$ as the maximizer of the value function $V(\pi,\lambda)$ that is associated with this reward,
\begin{align}\label{eqn_lagrangian_maximizers_rl} 
   \Pi(\lambda)  &=   
   \argmax_{\pi} \lim_{T\to\infty} \frac{1}{T} \E_{s,a\sim\pi} \left[ \sum_{t = 0}^T r_\lambda(s_t,a_t) \right]\nonumber\\
     & \triangleq   \argmax_{\pi} V(\pi,\lambda). \tag{R-CRL}
\end{align}
It is important for forthcoming discussions that $\Pi(\lam)$ is a set. There may be several policies that maximize value function $V(\pi,\lambda)$. This is the same observation we made of \eqref{P:constrainedRL}. We use $\pi(\lam)$ to denote individual Lagrangian maximizing policies and $\pi(s, \lam)$ to denote their evaluation at state $s$.

The drawback of using \eqref{eqn_lagrangian_maximizers_rl} is that appropriate choice of the penalty coefficients $\lam_i$ is difficult. Among other challenges, coefficients $\lam_i$ that make \eqref{eqn_lagrangian_maximizers_rl} and \eqref{P:constrainedRL} close depend on the transition probability of the MDP, which is assumed unknown. We can circumvent this problem by adapting the $\lam_i$ penalties. To do that, we introduce an iteration index $k$, a step size $\eta_\lambda$ and an epoch duration $T_0$. Coefficients $\lam_i$ are updated as determined by the iteration, 
\begin{align}\label{eqn_stochastic_dual_descent_specified}
   \lambda_{i, k+1} 
      = \left [  \lambda_{i,k} - 
           \frac{\eta_\lambda}{T_0}
              \sum_{t= kT_0}^{(k+1)T_0-1} 
                 \Big( r_i(s_t,a_t)-c_i \Big) 
                    \right]_+\hspace{-0.5ex},\hspace{-0.5ex} 
\end{align}
where actions $a_t \sim \pi(s_t, \lam_k)$ for any policy $\pi(\lam_k)\in\Pi(\lam_k)$. According to this expression, time is separated in epochs of duration $T_0$. During the $k$-th epoch, the penalty coefficients are $\lam_{k}$ and we consequently execute the policy $ \pi(\lam_k)$ between times $kT_0$ and $(k+1)T_0 -1$. During this epoch, the agent collects rewards $r_i(s_t,a_t)$ that are attributed to this policy. Ultimately, $T_0$ acts a an additional hyper-parameter whose optimal choice is dependent on the properties of the underlying MDP. Further discussion on the choice of $T_0$ can be found on Appendix \ref{app:sec_realizable}. The penalty coefficient $\lam_{i,k}$ is updated according to whether the accumulated reward exceeded the target $c_i$ or not. If the accumulated reward exceeds the target $c_i$, the penalty coefficient $\lam_{i,k}$ is reduced. If the accumulated reward falls below target, the coefficient is increased. Beyond a sensible update rule, \eqref{eqn_stochastic_dual_descent_specified} can be seen as a stochastic gradient descent update on the dual function. Leveraging this fact, this paper establishes the following theorem, the proof of which can be found in Appendix \ref{sec_unbiased}. 

\begin{theorem}\label{theo_main_intro}
An agent chooses actions according to $a_t \sim \pi(\lambda_k)$, where $\pi(\lam_k)\in\Pi(\lam_k)$ with $\lambda_k$ updates following \eqref{eqn_stochastic_dual_descent_specified}. Assume that
\begin{description}
\item[(A1)] There exists a strictly feasible policy $\pi^\dagger$ such that for some $C>0$ and all $i$ constraints $V_i(\pi^\dagger)-c_i \geq C$.
\item[(A2)] There exists $B$ such that $|r_i(s,a) - c_i| \leq B$ for all states $s$, actions $a$ and constraints $i$.
\item[(A3)] The accumulated reward in \eqref{eqn_stochastic_dual_descent_specified} is an unbiased estimate of the value function $V_i(\pi(\lam_k))$,
\begin{align}\label{eqn_unbiased}
   \E\left[
      \frac{1}{T_0} 
         \sum_{t= kT_0}^{(k+1)T_0-1} 
            r_i(s_t,a_t)
               ~\Big |~ \lam_k
                  \right] 
                     = V_i \bigl(\pi(\lam_k)\bigr).
\end{align}
\end{description}
Then, the state-action sequences $(s_t,a_t)$ are feasible with probability one,
\begin{align}\label{eqn_theo_feasibility_informal}  
    \liminf_{T\to\infty} \frac{1}{T}\sum_{t=0}^{T-1} r_i(s_t,a_t) \geq c_i, 
    \quad \text{~a.s.}
\end{align}
And they are near-optimal in the following sense
\begin{align}\label{eqn_theo_optimality_informal}
  \lim_{T\to\infty} \mathbb{E}\left[\frac{1}{T} \sum_{t=0}^{T-1} r_0(s_t,a_t) \right]\geq P^\star-\eta_\lambda \frac{B^2}{2}.
\end{align}
\end{theorem}
\begin{proof}
See Appendix \ref{sec_unbiased}.
\end{proof}

Theorem \ref{theo_main_intro} claims that if the penalty coefficients are \emph{continuously updated} as in \eqref{eqn_stochastic_dual_descent_specified} with actions being chosen as $a_t \sim \pi(\lam_k)$, then the state-action \emph{trajectory} satisfies the constraints in \eqref{P:constrainedRL} and attains a reward that is within $\eta_\lambda B^2 / 2$ of optimal. This claim produces an \emph{algorithm} whose execution results in actions that are feasible and nearly optimal in the long term. This is notable because \eqref{P:constrainedRL} is not a convex optimization problem. Although not completely unexpected since recent results have shown that \eqref{P:constrainedRL} has null duality gap \cite{paternain2019constrained}. And indeed, a result analogous to Theorem \ref{theo_main_intro} for discounted constrained MDPs is part of \cite{ding2020natural}. 

The assumptions made in Theorem \ref{theo_main_intro} are customary. Assumption~(A1) merely asks for the existence of a feasible policy, while (A2) requires the rewards to be bounded, a necessary condition for producing feasible iterates. Assumption~(A3) is guaranteed if one can sample from the stationary distribution of the MDP. A common hypothesis in RL, specially for continuing tasks \cite{sutton2018reinforcement}. However, it is possible that (A3) might not hold in practice. Appendix \ref{app:sec_realizable} addresses this, replacing (A3) in a  more technical variation of Theorem \ref{theo_main_intro}. Under a judicious choice of the time horizon $T_0$, the same feasibility guarantees are shown to hold, with only a reduction of optimality ensuing.

Hypotheses aside, there is however a caveat to Theorem \ref{theo_main_intro}. While notable for what it claims, it is also notable for what it does \emph{not} claim. Theorem \ref{theo_main_intro} does not guarantee that we are able to find an optimal policy $\pi^\star$. It is not possible to stop the iteration in \eqref{eqn_stochastic_dual_descent_specified} at some $k$ and claim that the policy $\pi(\lam_k)$ is close to optimal. This is a subtle yet fundamental issue because it precludes a claim that \eqref{eqn_stochastic_dual_descent_specified} finds the optimal policy. Instead, \eqref{eqn_stochastic_dual_descent_specified} generates \emph{trajectories} that are feasible and near-optimal by the \emph{continuous execution} of the primal iteration in \eqref{eqn_lagrangian_maximizers_rl} and the dual iteration in \eqref{eqn_stochastic_dual_descent_specified}. Section \ref{sec_discussions} covers a simple example to further clarify this point.

The algorithmic implication is that we must sample from the policies $\pi(\lam_k)$ in \eqref{eqn_lagrangian_maximizers_rl} while we run the dual iteration in \eqref{eqn_stochastic_dual_descent_specified}; see Algorithm~\ref{algorithm_exec}. Then, the purpose of a CRL algorithm can be modified as learning the policies $\pi(\lam_k)$ that solve \eqref{eqn_lagrangian_maximizers_rl}; see Algorithm~\ref{algorithm}. Both of these algorithmic modifications are easiest to understand as a form of state augmentation as we explain in the following section. 

\subsection{Augmented Constrained Reinforcement Learning} \label{sec_algorithm}

The \emph{State Augmented Constrained Reinforcement Learning (A-CRL)} algorithm reinterprets $\lam_k$ as part of the state space. This is possible because the dual variable update in \eqref{eqn_stochastic_dual_descent_specified} defines a Markov process. If we let $\mathbf{s}_k$ and $\mathbf{a}_k$ denote the states and actions that are observed during the $k$-th epoch, the update in  \eqref{eqn_stochastic_dual_descent_specified} defines a memoryless probability distribution,
\begin{align}\label{eqn_dual_pd}
   \lambda_{k+1} \sim
         p(\lam_{k+1} | \lam_k,  \mathbf{s}_k, \mathbf{a}_k). 
\end{align}
This state augmentation is summarized in Figure \ref{fig_blcok_diagram}. The top part of the diagram represents the transition probability of the given MDP. The bottom part of the diagram represents state augmentation with Lagrange multipliers that evolve according to the probability distribution in \eqref{eqn_dual_pd} defined by the update in \eqref{eqn_stochastic_dual_descent_specified}. Put together, we have a state \emph{augmented} MDP whose state is made of of the \emph{pair} $(s_t, \lam_k)$. Observe that the time scales at which $s_t$ and $\lam_k$ are updated are different. This is important in both theory (see Appendix) and practice (see Section \ref{sec_example}).

\begin{figure}[t]
    \centering
	\includegraphics[scale=1]{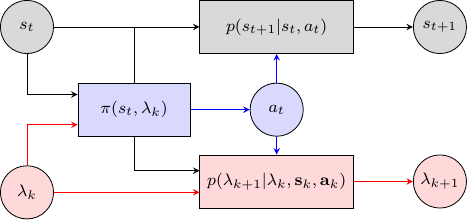} 					    			
	\caption{To solve CRL we must augment the state to incorporate multipliers $\lam_k$ (red). During training we learn a policy for this augmented MDP (blue). During online execution we must run dual dynamics (red) to guarantee convergence (Theorem \ref{theo_main_intro}).}
    \label{fig_blcok_diagram}
\end{figure}

In the middle of the diagram in Figure \ref{fig_blcok_diagram} we have the policy $\pi (s_t,\lam_k)$ which generates actions $a_t$. This policy is a function of the MDP's state and the multiplier $\lam_k$. Our goal is to learn a policy for which $\pi (\,\cdot\,, \lam) = \pi(\lam)$. That is, a policy $\pi$ that when evaluated at a specific state $\lambda$ yields a Lagrangian maximizing policy $\pi(\lam)$ [cf. \eqref{eqn_lagrangian_maximizers_rl}]. Remarkably, this is just a conventional RL problem with instantaneous rewards given by \eqref{eqn_lagrangian_reward}. Thus, the training step of A-CRL entails the use of any RL algorithm of choice to solve this augmented MDP. 

This training methodology is summarized in Algorithm \ref{algorithm}. For practical reasons we introduce a parameter $\theta \in \mathbb{R}^d$ and consider parameterized policies $\pi_\theta$ so that actions are drawn as $a\sim\pi_\theta(s,\lam)$. To train this policy we sample the augmented state space ${\cal S}\times\Lambda$ (Step 1) and evaluate the rewards $ r_\lambda(s_t,a_t)$ (Step 2). These rewards define an (augmented) MDP that we solve with the chosen RL method (Step 3). Since the rewards used in Step 3 are the ones in \eqref{eqn_lagrangian_reward}, the policy $\pi_{\theta^\star}$ is such that when evaluating $\pi_{\theta^\star}(\,\cdot\, , \lam)$ for a given $\lambda$ the policy is a Lagrangian maximizing policy $\pi(\lambda)$ if the learning parametrization is sufficiently expressive (in the sense of universal approximators such as neural networks \cite{hornik1989multilayer} or radial basis functions \cite{park1991universal}). Classic reinforcement learning algorithms, such as policy gradient \cite{sutton2000policy} or actor-critic methods \cite{konda2000actor} can be used. These are gradient-based algorithms which are not guaranteed to converge to global optimality (unless the state-action space is finite \cite{agarwal2020optimality}), yet they show good empirical success converging to solutions with small suboptimality. We point out that numerical results in Section \ref{sec_example} use policy gradient in Step 3, but A-CRL is agnostic to this choice.

\begin{algorithm}[t]
{
\renewcommand{\algorithmicrequire}{\textbf{Input:}}
\renewcommand{\algorithmicensure}{\textbf{Output:}}
\caption{A-CRL Training Phase} \label{algorithm} 
\begin{algorithmic}[1]  
\Ensure Trained policy $\pi_{\theta^\star}$
\State Sample $(s,\lam)$ from augmented space $\calS\times\Lambda$ 
\State Construct the augmented rewards [cf. \eqref{eqn_lagrangian_reward}]
$$\displaystyle{
 r_\lambda(s,a) = r_0(s,a) + \sum_{i=1}^m \lam_i \Big( r_i(s,a) - c_i \Big)}
$$
\State Use the RL algorithm to obtain policy
$$\displaystyle{
\theta^\star  =  
\argmax\limits_{\theta \in \setR^d} \lim\limits_{T\to\infty} \frac{1}{T}\E_{s,a\sim\pi_\theta} \left[ \sum_{t = 0}^T r_{\lambda}(s_t,a_t)\right]}
$$
\end{algorithmic}}
\end{algorithm}

\begin{algorithm}[t]
{
\renewcommand{\algorithmicrequire}{\textbf{Input:}}
\renewcommand{\algorithmicensure}{\textbf{Output:}}
\caption{A-CRL Execution Phase} \label{algorithm_exec} 
\begin{algorithmic}[1]  
\Require Policy $\pi_{\theta^\star}$, step $\eta_\lambda$, requirements $c_i$, epoch $T_0$
\Ensure Trajectories $(s_t,a_t)$ satisfying Theorem \ref{theo_main_intro}
\setcounter{ALG@line}{0}
\State \textit{Initialize}: Given initial state $s_0$, dual variable $\lambda_0 = 0$
\For {$k=0,1\ldots$}
\State Rollout $T_0$ steps with actions $a_t\sim\pi_{\theta^\star}(s_t,\lambda_k)$
\State Update dual dynamics [cf. \eqref{eqn_stochastic_dual_descent_specified}]
$$
\displaystyle{\lambda_{i, k+1} = 
   \left [ \lambda_{i,k} - 
      \frac{\eta_\lambda}{T_0}  
         \sum_{t= kT_0}^{(k+1)T_0-1} 
            \Big( r_i(s_t,a_t)-c_i \Big)
               \right]_+} 
$$
\EndFor
\end{algorithmic}}
\end{algorithm}

The more important aspect to highlight of A-CRL is that during execution, we use the trained policy $\pi_{\theta^\star}(s, \lambda)$, and continuously sample $\lambda_k$ by using the dual dynamics in  \eqref{eqn_stochastic_dual_descent_specified}. This is summarized in Algorithm \ref{algorithm_exec}. Time is divided in epochs of duration $T_0$ (Step 2). During each epoch we rollout $T_0$ steps where actions are drawn as  $a_t\sim\pi_{\theta^\star}(s_t,\lambda_k)$ (Step 3). We then proceed to update the state of the dual variables (Step 4). This algorithm is unusual because we do not use \eqref{eqn_stochastic_dual_descent_specified} as part of a learning process. We use \eqref{eqn_stochastic_dual_descent_specified} as an online computation that controls a switching across the policies $\pi_{\theta^\star}(\,\cdot\, ,\lambda_k)$ that are learned during the training process. Online execution of Algorithm \ref{algorithm_exec} guarantees that the claims of Theorem \ref{theo_main_intro} hold for the MDP trajectory $(s_t,a_t)$.

A-CRL redefines the purpose of CRL as learning the Lagrangian maximizing policies $\pi(\lam)$. This is a good premise because we know how to find policies $\pi(\lam)$. It is just an MDP with an augmented state space that we solve with Algorithm 1. We further know that executing policies $\pi(\lam_k)$ while updating $\lam_k$ with the dynamics in \eqref{eqn_stochastic_dual_descent_specified} guarantees the creation of trajectories that satisfy the feasibility and near-optimality guarantees of Theorem \ref{theo_main_intro}. This is what Algorithm 2 does. One may think that this is a cumbersome redefinition of the stated goal of finding an optimal policy $\pi^\star$. However, as the next section explains, this is not an artifact but a fundamental aspect of CRL. 

\section{Discussions and Implications}\label{sec_discussions}

To explain the meaning of Theorem \ref{theo_main_intro} and the algorithm in Section \ref{sec_algorithm} it is instructive to go back to the monitoring task introduced in Figure \ref{fig:example}. We have a moving agent that must spend some of its time in each of two regions $R_1$ and $R_2$. A minimal abstraction of this problem is the MDP with 3 states that is also illustrated in Figure \ref{fig:example} and which we formally define next. 

\begin{example}[\textit{Monitoring Problem}]\label{ex_monitoring}
Consider an MDP with states $R_0$, $R_1$ and $R_2$. In each state the choice of action determines the next state, $s_{t+1} = a_t$. When in state $R_0$ possible actions are $\{R_1, R_2\}$. When in sate $R_i$ with $i\neq0$, possible actions are $\{R_0, R_i\}$. We want to spend at least $c_i$ fraction of time in each of the states $R_1$ and $R_2$. Otherwise, we want to spend as much time as possible at $R_0$. We set $c_1=c_2=c=1/3$.
\end{example}

To learn policies that solve the monitoring task in Example~\ref{ex_monitoring} we define the reward functions
\begin{align}\label{eqn_monitoring_rewards}
   r_i(s_t, a_t) 
      = \indicator 
         ( s_t = R_i ) 
            \text{\quad for~} i=0, 1, 2,
\end{align}
to formulate a constrained MDP [cf. \eqref{eqn_value_function}] in which we have: (i) Two constraints $V_i(\pi)\geq c_i$ for $i=1,2$ where the value functions $V_i(\pi)$ are associated with the corresponding rewards $r_i(s_t, a_t)$ of \eqref{eqn_monitoring_rewards}. (ii) The objective $V_0(\pi)$ is the value function associated with the reward $r_0(s_t, a_t)$ of \eqref{eqn_monitoring_rewards}. 

The choice $c=1/3$ makes an optimal policy $\pi^\star$ that solves this constrained MDP ready to find. Just choose with probability $1/2$ between the two actions that are allowed at each state. This yields a symmetric Markov chain in which the ergodic limits are $1/3$ for all states. This satisfies the requirement of spending at least $c=1/3$ fraction of time at $R_1$ and $R_2$ and maximizes the time spent at $R_0$. More to the point, this policy is optimal because $V_1(\pi^\star)\geq c=1/3$, $V_2(\pi^\star)\geq c=1/3$ and $V_0(\pi^\star) = 1/3 = P^\star$.

To design an algorithm that finds a solution we write the Lagrangian [cf. \eqref{eqn_lagrangian}] to transform the constrained MDP into a plain MDP. As we saw in \eqref{eqn_lagrangian_reward}, doing so is equivalent to defining a weighted reward as in \eqref{eqn_lagrangian_maximizers_rl}. In this particular example, the weighted reward takes the form
\begin{align}\label{eqn_monitoring_weighted_reward}
   r_{\lambda} (s_t, a_t)
      = \indicator ( s_t = R_0 ) 
         &+ \lambda_1 \Big(\indicator ( s_t = R_1 ) - c \Big) \nonumber\\
         &+ \lambda_2 \Big( \indicator ( s_t = R_2 ) - c \Big).
\end{align}
An unconstrained MDP with this reward has elementary solutions. If $1>\lam_1$ and $1>\lam_2$, we maximize the accumulated reward $(1/T)\sum_{t = 0}^T r_{\lambda}(s_t, a_t)$ by staying at the region $R_0$ as much as possible; which is half of the time. If $\lam_1>1$ and $\lam_1>\lam_2$, we maximize the accumulated reward by staying at $R_1$ all of the time. Similarly, if $\lam_2>1$ and $\lam_2>\lam_1$, we maximize the accumulated reward by staying at $R_2$ all of the time. Clearly, none of these three policies are optimal for the constrained MDP, as they do not satisfy the constraints. The remaining option is to make $\lam_1=\lam_2=1$. This choice of weights is such that \emph{any} policy maximizes the accumulated reward of the unconstrained MDP, since the agent collects the same reward in all states. The set $\Pi(\lam_1=1, \lam_2=1)$ includes all policies. We therefore see that a constrained MDP that uses the rewards in \eqref{eqn_monitoring_rewards} as constraints and its penalized version using the reward in \eqref{eqn_monitoring_weighted_reward} are not equivalent. This exemplifies the limitations of learning with rewards. We summarize this observation in the following remark.

\begin{remark}\label{rmk_cmdp_pmdp_not_equivalent}
The constrained MDP in \eqref{P:constrainedRL} and the regularized MDP in \eqref{eqn_lagrangian_maximizers_rl} are not equivalent representations of the same problem. There exist constrained MDPs for which there is no unconstrained MDP with regularizer $\lambda$, such that $\Pi^\star \equiv \Pi(\lambda)$.
\end{remark}

Remark \ref{rmk_cmdp_pmdp_not_equivalent} appears to indicate that the algorithm defined by~\eqref{eqn_stochastic_dual_descent_specified} should not be able to find solutions of the constrained MDP defined in \eqref{P:constrainedRL}. Any choice of $\lambda$ results in an optimal set $\Pi(\lambda)$ that is not equivalent to the optimal set $\Pi^\star$. Whenever we solve \eqref{eqn_lagrangian_maximizers_rl} we obtain a policy that is not optimal. Further note that, if anything, Remark \ref{rmk_cmdp_pmdp_not_equivalent} understates the magnitude of the problem. The policies $\pi(\lam)$ for the monitoring problem bear no resemblance to the optimal policy $\pi^\star$ unless  $\lam_1=\lam_2=1$. And when $\lam_1=\lam_2=1$ all policies belong to the set $\Pi(\lam_1=1, \lam_2=1)$. One of these is the optimal policy $\pi^\star$. This is a tautology that provides no information about how to find the optimal policy. 

How come, then, that Theorem \ref{theo_main_intro} holds? As we explained in Section \ref{S:constrainedRL}, Theorem \ref{theo_main_intro} does not claim that we find the optimal policy. Rather, that we can provide an \emph{algorithm} whose execution results in the solution of \eqref{P:constrainedRL}. This is weaker, but still difficult to reconcile with the statement of Remark \ref{rmk_cmdp_pmdp_not_equivalent}. To reconcile these statements, we dwell in the duality properties of the constrained MDP in \eqref{P:constrainedRL}.

\begin{figure*}[t!]
    \centering
    \subfigure[{Constraint satisfaction $V_1(\pi)$.}]{
	\includegraphics[scale=1.075]{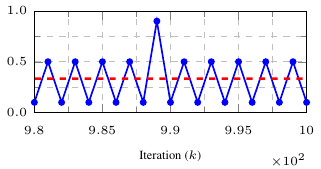} 					
     \label{fig:example_sim_dual_v}
	 } \hspace{-0.325cm}
    \subfigure[{Dual variable $\lambda_{1,k}$.}]{
	\includegraphics[scale=1.075]{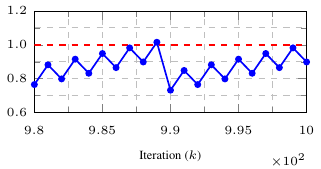} 					
     \label{fig:example_sim_dual_dual}
	 } \hspace{-0.325cm}
    \subfigure[{Policy $\pi(s_t=R_1, \lam_k)$.}]{
	\includegraphics[scale=1.075]{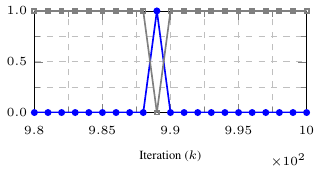} 					
     \label{fig:example_sim_dual_policy}
	 }
	\caption{Execution of A-CRL for Example \ref{ex_monitoring}. Values for average accumulated reward $r_1$ and dual variables $\lam_1$ are shown in Figures \ref{fig:example_sim_dual_v} and \ref{fig:example_sim_dual_dual}.
	Figure \ref{fig:example_sim_dual_policy} shows the policy $\pi(s_t=R_1, \lam_k)$. Probability of staying at $R_1$ (blue circles) and the complementary probability of jumping to $R_0$ (gray squares) are shown. We run $1{,}000$ rollouts with length $T_0=10$ and step size $\eta_\lambda=0.5$. Figures show last 20 rollouts. A-CRL does not converge to $\pi^\star$ but $\pi(\lam_k)$ generates a trajectory that is feasible [cf. \eqref{eqn_theo_feasibility_informal}] and nearly optimal [cf. \eqref{eqn_theo_optimality_informal}].
	}	    
	\label{fig:example_sim_dual} 
\end{figure*}

\subsection{Constrained Reinforcement Learning and Duality} \label{sec_relationship}

A candidate culprit for the observation in Remark \ref{rmk_cmdp_pmdp_not_equivalent} is lack of convexity. The functions involved in the constrained MDP in \eqref{P:constrainedRL} are not concave in the policy $\pi$. We will see here that this is not a problem because despite their lack of convexity, constrained MDPs have null duality gap~\cite{paternain2019constrained}. Formally, recall the definition of the dual function $d(\lambda)$ as the maximum of the Lagrangian 
\begin{align}\label{eqn_crl_dual_function}
   d(\lambda) & \triangleq \max_\pi \calL(\pi,\lambda) 
   = \calL(\pi(\lambda), \lambda),
\end{align}
where the Lagrangian maximizers are denoted by $\pi(\lambda)$ [cf. \eqref{eqn_lagrangian_maximizers} and \eqref{eqn_lagrangian_maximizers_rl}]. By its definition, for any $\lambda \in \setR^m_+$, the dual function provides an upper bound on the optimal value~$P^\star$ of \eqref{P:constrainedRL}. The dual problem aims to find the tightest of these bounds by minimizing the dual function,
\begin{align}\label{eqn_crl_dual_problem}\tag{D-CRL}
   \lambda^\star 
       \in \argmin_{\lambda \in \setR^m_+}  d(\lambda)
      = \argmin_{\lambda \in \setR^m_+}  \calL(\pi(\lambda), \lambda),    
\end{align}
and wherein the optimal dual value is denoted by $D^\star \triangleq d(\lambda^\star)$. In convex optimization problems we know that the primal and dual problems are equivalent in the sense that they attain the same optimal values. The constrained MDP in \eqref{P:constrainedRL} is not a convex optimization problem. Still, it has a particular structure that renders it equivalent to \eqref{eqn_crl_dual_problem} as we formally assert next. 

\begin{theorem}[Strong Duality \cite{paternain2019constrained}]\label{theo_no_gap}
The constrained reinforcement learning problem in \eqref{P:constrainedRL} and its dual form in \eqref{eqn_crl_dual_problem} are such that 
\begin{align}
    V_0(\pi^\star) = P^\star = D^\star = d(\lambda^\star).
\end{align} \end{theorem}

The equivalence of Theorem \ref{theo_no_gap} implies that the primal and dual \emph{optimal values} of the constrained MDP in \eqref{P:constrainedRL} are equal to the optimal value of the  unconstrained Lagrangian form in \eqref{eqn_lagrangian_maximizers_rl}. Thus, at a fundamental level, it would appear that constrained RL is not more difficult to solve than an unconstrained RL. One need only find the optimal Lagrange multiplier $\lambda^\star$ that renders \eqref{P:constrainedRL} and \eqref{eqn_lagrangian_maximizers_rl} equivalent and proceed to solve the unconstrained MDP in \eqref{eqn_lagrangian_maximizers_rl} with $\lambda=\lambda^\star$. This is a relatively simple task because the dual problem in \eqref{eqn_crl_dual_problem} is a convex program \cite{boyd2004convex} and therefore easy to minimize. Thus, solving a constrained reinforcement learning problem should not be more difficult than conventional reinforcement learning save for the relatively simple task of finding optimal dual variables. However, as we saw in Example \ref{ex_monitoring} and concluded in Remark \ref{rmk_cmdp_pmdp_not_equivalent}, \emph{this is not the case}. The optimal dual variable may be uninformative about the optimal policy. We explain why in the next section.

\subsection{Recovering Optimal Policies from Dual Variables}\label{sec_primal_recovery}

The key nuance that limits the use of the equivalence between \eqref{P:constrainedRL} and \eqref{eqn_lagrangian_maximizers_rl} that is implied by Theorem \ref{theo_no_gap} is this: Our goal is not to determine the optimal value $P^\star$ but to find an optimal policy $\pi^\star\in\Pi^\star$. The latter is not always easy to recover from \eqref{eqn_lagrangian_maximizers_rl} even if the optimal dual variable $\lambda^\star$ is known. To explain the complications that may arise, we introduce a proposition that relates optimal policies of \eqref{P:constrainedRL} to optimal policies of \eqref{eqn_lagrangian_maximizers_rl}.

\begin{proposition}\label{prop_primal_recovery} The set $\Pi^\star$ of optimal policies of \eqref{P:constrainedRL} is included in the set $\Pi(\lam^\star)$ of Lagrangian maximizing policies of \eqref{eqn_lagrangian_maximizers_rl} with $\lam=\lam^\star$,
\begin{align}\label{eqn_set_maximizers}
   \Pi^\star \subseteq \Pi(\lam^\star).
\end{align}
\end{proposition}
\begin{proof}
See Appendix \ref{app:prop_primal_recovery}.
\end{proof}
As per Proposition \ref{prop_primal_recovery}, the issue is that $\Pi(\lam^\star)$ is a set. Optimal policies $\pi^\star\in\Pi^\star$ are included in this set. But the two sets are not equivalent. There may be policies in $\Pi(\lam^\star)$ that are not optimal. This is the case of Example \ref{ex_monitoring} in which the set $\Pi(\lam^\star)$ contains all policies. This illustrates that, in some cases, finding a policy in the set $\Pi(\lam^\star)$ that is optimal for \eqref{P:constrainedRL} is as difficult as solving \eqref{P:constrainedRL} itself. The two problems are equivalent in this particular example.

An optimal policy $\pi^\star$ is readily recoverable from $\Pi(\lam^\star)$ when the latter set is a singleton. In such case \eqref{eqn_set_maximizers} implies the equivalence $\Pi^\star \equiv \Pi(\lam^\star)$ and policies $\pi(\lam_k)$ approach $\pi^\star$ when $\lam_k$ is updated as per \eqref{eqn_stochastic_dual_descent_specified}. The A-CRL algorithm is not necessary in this case, although it would still work. 

However, given that there is nothing contrived about Example \ref{ex_monitoring} it is reasonable to expect that problems where $\Pi(\lam^\star)$ is not a singleton are widespread. This expectation can be asserted by recalling that: (i) unconstrained MDPs \emph{always} admit a deterministic policy $\pi$ as optimal \cite{puterman1994markov}; (ii) constrained MDPs \emph{often} admit nondeterministic policies as solutions \cite{altman1999constrained}. The only way to reconcile the existence of a deterministic policy in the set $\Pi(\lam^\star)$ and the presence of a nondeterministic optimal policy $\Pi^\star$ with Proposition \ref{prop_primal_recovery} is for the set $\Pi(\lam^\star)$ to contain multiple policies. In these situations, A-CRL is necessary to find solutions of \eqref{P:constrainedRL}, in the sense of generating trajectories that satisfy \eqref{eqn_theo_feasibility_informal} and \eqref{eqn_theo_optimality_informal}. We illustrate this point next.

\begin{figure*}[t!]
    \centering
    \subfigure[{Constraint satisfaction $V_1(\pi)$.}]{
	\includegraphics[scale=1.075]{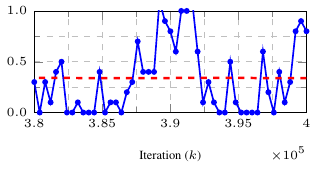} 					
     \label{fig:example_pd_v}
	 }
    \subfigure[{Dual variable $\lambda_{1,k}$.}]{
	\includegraphics[scale=1.075]{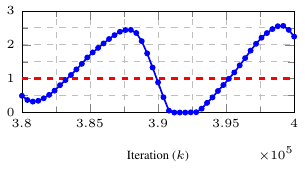} 					
     \label{fig:example_pd_dual}
	 }
    \subfigure[{Policy $\pi(s_t=R_1, \lam_k)$.}]{
	\includegraphics[scale=1.075]{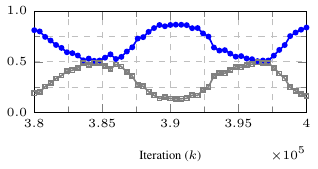} 					
     \label{fig:example_pd_policy}
	 }
	 
	\caption{Primal-dual method for Example \ref{ex_monitoring}. Figures \ref{fig:example_pd_v}\textendash\ref{fig:example_pd_policy} are analogous to Figures \ref{fig:example_sim_dual_v}\textendash\ref{fig:example_sim_dual_policy}. We run $400{,}000$ rollouts with length $T_0=10$, dual step $\eta_\theta=0.0025$ and primal step $\eta_\theta=0.025$. Primal-dual does not learn the optimal policy although it generates \emph{trajectories} that are optimal on average if executed online, as is the case of A-CRL. Executing primal-dual methods online is challenging.}	    
    \label{fig:example_sim}
\end{figure*}

\subsection{Policy Switching and Memory}\label{sec_policy_switching}

\begin{figure*}[t!]
    \centering
    \subfigure[{Policy for $\lambda=[5,0,0,0]$.}]{
	\includegraphics[scale=1.075]{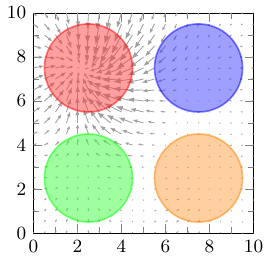} 	
     \label{fig:quiver_high_low}
	 }
    \subfigure[{Policy for $\lambda=[0,5,5,0]$.}]{
	\includegraphics[scale=1.075]{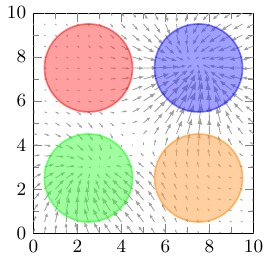} 		
     \label{fig:quiver_mid_mid}
	 }
    \subfigure[Occupation measure.]{
	\includegraphics[scale=1.075]{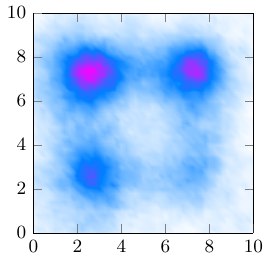} 			
     \label{fig:heatmap}
	 }	 
	\caption{Policy $\pi_{\theta}(\lambda)$ after $1{,}000{,}000$ training iterations, shown for two possible values of $\lambda$. The state occupation measure resulting from an agent executing $\pi_{\theta}(\lambda)$ coupled with the dual dynamics \eqref{eqn_stochastic_dual_descent_specified} is shown on the right.}
    \label{fig:quiver}
\end{figure*}

In the training stage of A-CRL, Algorithm \ref{algorithm} solves the augmented MDP of Figure \ref{fig_blcok_diagram}. Assuming that training succeeds, this is equivalent to finding Lagrangian maximizing policies $\pi(\lam)$. As explained in Section \ref{sec_primal_recovery} there is no guarantee that the optimal policy $\pi^\star$ is recoverable from $\pi(\lam^\star)$. Thus, no guarantee that A-CRL converges to $\pi^\star$ either. Nevertheless, A-CRL guarantees that \eqref{eqn_theo_feasibility_informal} and \eqref{eqn_theo_optimality_informal} hold. This is attained through \emph{policy switching.} Step 4 of Algorithm \ref{algorithm_exec} updates multipliers as per \eqref{eqn_stochastic_dual_descent_specified} in order to effect the execution of different policies  $\pi(\lam_k)$ in Step 3.

 Figure \ref{fig:example_sim_dual} illustrates policy switching for Example \ref{ex_monitoring}, showing variables of A-CRL related to state $R_1$. Figure \ref{fig:example_sim_dual_v} shows the average accumulated reward $r_1$ and Figure \ref{fig:example_sim_dual_dual} shows the value of the Lagrange multiplier $\lam_1$. Figure \ref{fig:example_sim_dual_policy} illustrates the policy $\pi(s_t=R_1, \lam_k)$. It shows the probability of staying at $R_1$ (blue circles) and the complementary probability of jumping to $R_0$ (gray squares). In this latter figure, policy $\pi(s_t=R_1, \lam_k)$ jumps from staying at $R_1$ with probability $0$ to staying at $R_1$ with probability $1$. None of these two policies solve Example \ref{ex_monitoring}. This fact is consistent with the impossibility of recovering the optimal policy from any given set of multipliers (Section \ref{sec_primal_recovery}). This is also something we have already highlighted in Remark \ref{rmk_cmdp_pmdp_not_equivalent}. Maximizing \eqref{eqn_monitoring_weighted_reward} requires maximizing the time spent at a particular state $R_i$ depending on the values of the multipliers $\lam_i$. 
 
That the policies in Figure \ref{fig:example_sim_dual_policy} do not solve Example \ref{ex_monitoring} is attested by the accumulated reward plot (solid blue) in Figure \ref{fig:example_sim_dual_v}. The requirement is for this reward to accumulate to at least $c=1/3$ and the optimal policy requires that this reward accumulate to exactly $c=1/3$. This does not happen at any iteration. However, the time average of the accumulated reward across all rollouts does accumulate to exactly $1/3$ (dashed red). This is as it should be. It is the claim of Theorem \ref{theo_main_intro} that this time average should satisfy the constraints with probability $1$ if we execute A-CRL.

Figure \ref{fig:example_sim_dual_dual} illustrates how the Lagrange multiplier controls the policy switch. As it follows from \eqref{eqn_monitoring_weighted_reward}, when $\lam_1<1$  the Lagrangian maximizing policy requires jumping back to $R_0$ with probability $1$. When the multiplier switches to $\lam_1>1$ the Lagrangian maximizing policy switches to one that requires staying at $R_1$ with probability $1$\textemdash we must also have $\lam_2>\lam_1$, this holds but is not shown. In Figure \ref{fig:example_sim_dual_dual} the switching from $\lam_1<1$ to $\lam_1>1$ happens for $k=989$. This is precisely the same moment when the policies switch in Figure \ref{fig:example_sim_dual_policy}. This highlights the instrumental role of state augmentation in solving CRL problems. It is the policy switching controlled by $\lam_i$ that generates trajectories that guarantee feasibility and near-optimality (Theorem \ref{theo_main_intro}).

\textbf{Memory.~} The Lagrange multipliers $\lambda_{i, k}$ in A-CRL accept an interpretation in terms of memory. To make this clearest assume that $\lambda_{i, 0} = 0$ and that the nonnegative projection operators in \eqref{eqn_stochastic_dual_descent_specified} is never applied\textemdash which is true in Figure \ref{fig:example_sim_dual_dual}. With these assumptions we can rewrite \eqref{eqn_stochastic_dual_descent_specified} as
\begin{align}\label{eqn_dd_ex1}
   \lambda_{i, k} 
      = \frac{\eta_\lambda}{T_0}
           \sum_{t= 0}^{kT_0-1} 
              \Big(c - \indicator ( s_t = R_i )\Big) .
\end{align}
Thus, the multiplier $\lam_i$ records the deficit in the amount of time the agent spends at $R_i$ relative to the requirement $c$. When this deficit grows large, the multiplier triggers a policy switch to reduce it. This observation holds true in general. The multiplier update in \eqref{eqn_stochastic_dual_descent_specified} is such that $\lam_{i,k}$ memorizes the constraint deficit. As these deficits rise and fall, they control the choice of policies $\pi(\lam_k)$ to generate trajectories that satisfy \eqref{eqn_theo_feasibility_informal} and \eqref{eqn_theo_optimality_informal}. The nonnegative projection operator limits the future credit that an agent is given for over-satisfying the constraint at a particular rollout. 

\textbf{Primal Averaging. } An alternative but rather inefficient means to solve \eqref{P:constrainedRL} would be primal averaging. The optimal policy $\pi^\star$ can be recovered from primal averaging \cite{anstreicher2009two,nedic2009approximate}. Indeed, the average policy
\begin{align}\label{eqn_primal_averaging}
    \bar\pi_K  = \frac{1}{K} \sum_{k=0}^{K-1} \pi(\lam_k),
\end{align}
consists of the combination up to a time horizon $K$ of the Lagrangian maximizing policies sampled while running dual dynamics \eqref{eqn_stochastic_dual_descent_specified}. If the horizon is sufficiently large, drawing actions from this average policy generates a sequence with the same ergodic properties of the sequence generated by A-CRL. Since the guarantees of Theorem \ref{theo_main_intro} hold for the latter, they also hold for a trajectory that is sampled from $\bar\pi_K$. Primal averaging provides an alternative to A-CRL that is applicable to some problems but not all. In general, policies $\pi(\lam_k)$ are parametrized as in Algorithms \ref{algorithm} and \ref{algorithm_exec}. However, \eqref{eqn_primal_averaging} is averaged over \emph{policies}, not \emph{parameterizations}. Finding a parametric description of the average policy $\bar\pi_K$ other than the impractical storage of the whole sequence of parameters is a difficult problem in itself.

\subsection{Primal-Dual Methods}\label{sec_primal_dual}

A-CRL uses gradient descent in the dual domain and policy optimization in the primal domain. Primal-dual methods use gradient descent in the dual domain and gradient ascent in the primal domain. That is, the maximization step in \eqref{eqn_lagrangian_maximizers_rl} is replaced by the gradient ascent step,
\begin{align}\label{eq:pd_equations}
   \theta_{k+1}&= \theta_{k} + \eta_\theta  \widehat{\nabla}_{\theta} \calL(\theta_k,\lambda_k) .
\end{align}
The trajectories generated by primal-dual methods solve CRL in the same time average sense of Theorem \ref{theo_main_intro}\cite{tessler2018reward, ding2020natural}. This is illustrated in Figure \ref{fig:example_sim} where we show the same variables of Figure \ref{fig:example_sim_dual}  for a primal-dual method. The same switching behavior is apparent. Multipliers rise and fall [Figure \ref{fig:example_pd_dual}] triggering policy switches [Figure \ref{fig:example_pd_policy}]. These policies are instantaneously not optimal, but they satisfy the constraints on average [Figure \ref{fig:example_pd_v}]. The important consequence is that, in general, we cannot stop a primal-dual method at iteration $k$ and claim that the policy $\pi_{\theta,k}$ is close to optimal. As is the case of A-CRL, the primal-dual method must be run online to ensure that \emph{trajectories} are optimal on average.

While possible, running \eqref{eq:pd_equations} online is challenging. This is because online estimation of the gradient in \eqref{eq:pd_equations} is difficult. We could attempt to mitigate this problem with an actor-critic method, but one has remember that the usefulness of actor-critic hinges upon convergence to a stationary policy\textemdash something that does not hold for some CRL problems. A-CRL circumvents these challenges by augmenting the state and learning instead the Lagrangian maximizing policies in \eqref{eqn_lagrangian_maximizers_rl}.

We also call attention to the significant delay between the time at which $\lam_1>1$ [$k\approx 3.82 \times 10^5$ in Figure \ref{fig:example_pd_dual}] and the time at which the policy switches [$k\approx 3.90 \times 10^5$ in Figure \ref{fig:example_pd_policy}]. This is because as $\lam_1$ switches, the primal-dual method needs to relearn the Lagrangian maximizing policy. This delay results on $\lam_1$ growing larger than in Figure \ref{fig:example_sim_dual}. As it follows from \eqref{eqn_dd_ex1} this implies that the transient violation of constraints is larger for primal-dual methods than it is for A-CRL. Even if both satisfy the constraints on average. 

\section{Numerical Example}\label{sec_example}

Let us consider again the \emph{monitoring problem}, first introduced in Fig.~\ref{fig:example}, in which an agent must spend a specific portion of its time in certain regions of the environment. As we discussed in the simpler Example~\ref{ex_monitoring}, this is a problem for which classical reinforcement learning is unable to learn a solution. For the numerical results in this section, we study a slightly more complex version of this problem. More formally, let us consider a MDP formed by a state space $\mathcal{S}=[0,10] \times [0,10]$ representing the $x$- and $y$-axis position of the agent. The state of the agent evolves following the dynamics $s_{t+1}=s_t+T_s a_t$, where $T_s=0.05$ is the chosen sampling time of the system. The agent samples actions $a_t$ from a Gaussian policy $\pi_\theta$, the mean of which is constructed by a function approximator given by a weighted linear combination of Gaussian kernels. We choose the bandwidth of the Gaussian kernels to be equal to their spacing, with them being spaced 1 unit from each other. The weights of this function approximator are the parameters to be learned by the agent. The task is specified as a CRL problem. Specifically, setting $r_0(s_t,a_t)=0$ and $r_i(s_t,a_t)=1$ if the agent is inside the $i$-th region to be monitored. As illustrated in Figure \ref{fig:quiver}, we consider four regions to be monitored, indexed by $i\in \{1,2,3,4\}=\{ \texttt{red}, \texttt{blue}, \texttt{green}, \texttt{orange}\}$, where the colors are matched in the plots. The task requirements $c_i$ of the agent are to monitor each region at least $20\%$, $15\%$, $10\%$ and $5\%$ of its time, respectively. The agent is trained via the A-CRL algorithm described in Section \ref{sec_algorithm}. During the training phase, the RL algorithm used is a policy gradient \cite{sutton2000policy} with rollouts of fixed horizon $T=20$ and step size of $\eta_{\theta}=0.001$, for $1{,}000{,}000$ iterations.

\begin{figure}[t!]
    \centering
    \subfigure[Time average value of the dual variables.]{
    \hspace{2.15ex}
	\includegraphics[scale=1]{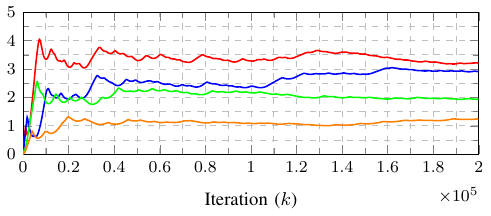} 			
     \label{fig:dual_variables}
    }
    \subfigure[Instantaneous value of the dual variables (last $10{,}000$ iterations).]{
    \hspace{2.15ex}
	\includegraphics[scale=1]{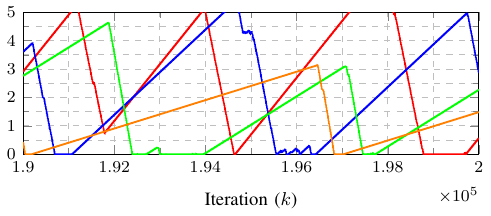} 			
     \label{fig:dual_variables_instant}
    }    
    \subfigure[Constraint satisfaction.]{
	\includegraphics[scale=1]{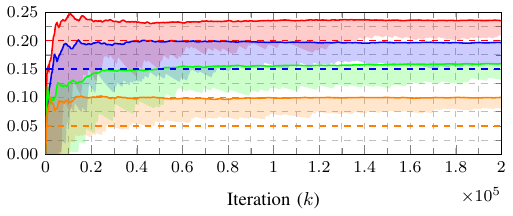} 				
     \label{fig:constraints}
	 }	 
	\caption{A-CRL algorithm. Dual variables and constraint satisfaction when executing the trained policy $\pi_{\theta}(\lambda)$ coupled with the dual dynamics \eqref{eqn_stochastic_dual_descent_specified}. The dual variables are shown for a single realization. The constraint satisfaction is averaged (solid line) over $100$ independent runs, where the shadowed band covers to the worst realization. Requirements are shown as dashed lines and the colors are set accordingly to the region they represent, following Fig. \ref{fig:quiver}.}
    \label{fig:dual_constraints}
\end{figure}

The policy $\pi_\theta(\lambda)$ after training is shown in Figure \ref{fig:quiver}. Specifically, Fig. \ref{fig:quiver_high_low} and Fig. \ref{fig:quiver_mid_mid} show the resulting policy for different values of the Lagrange multipliers $\lambda$. In the A-CRL method, the policy includes the Lagrangian variables as part of the augmented state space. When the Lagrange multiplier associated to one of the regions takes a higher value compared to the other regions, the policy moves the agent towards that region. This is the case shown in Fig. \ref{fig:quiver_high_low}, where all dual variables but $\lambda_1=5$ are set to zero. For such case, the policy is to move towards the red ($i=1$) monitoring area. Similarly, we show in Fig. \ref{fig:quiver_mid_mid} the policy when two of the Lagrangian multipliers take large values, and the other regions are kept to zero. In this case, the policy is to move into any of the two regions, blue ($i=2$) or green ($i=3$). Similar behavior can be observed when setting the Lagrange multipliers of the other regions accordingly. 

It is important to remember that no dual dynamics are used during training. Recall that the policy acquired during training maximizes the Lagrangian of the problem [cf. \eqref{eqn_lagrangian}] for all possible Lagrange multipliers. Now, consider an agent executing this policy. We use the requirements $c_i$ and generate Lagrange multipliers $\lambda_{i,k}$ following the dual dynamics, as in \eqref{eqn_stochastic_dual_descent_specified}, using a step size $\eta_{\lambda}=0.01$ and an epoch length of $T_0=1$. We compute the state occupation measure shown in Fig. \ref{fig:heatmap}. When executing the A-CRL algorithm, the agent spends more in in areas with higher demands, while barely spending any time in the areas with no demands (outside of the monitoring regions).

A more detailed understanding regarding the underpinnings of the A-CRL approach can be attained by examining the values of the dual variables and the constraint satisfaction (the percentage of time spent in each of the monitoring areas). These are shown in Figure \ref{fig:dual_constraints}. For the dual variables, we show in Fig. \ref{fig:dual_variables} their time averages for a single run, and in Fig. \ref{fig:dual_variables_instant} their instantaneous values during the last $10{,}000$ iterations, wherein the switching behavior discussed in \ref{sec_policy_switching} can be observed. Requirements more difficult to satisfy lead to higher values of their respective dual variables. More precise information is shown by the constraint satisfaction, plotted in Fig. \ref{fig:constraints}. In this case, we plot the results of $100$ independent realizations, each one over $200{,}000$ time steps. The monitoring requirements are shown by a dashed line, the average constraint satisfaction by a solid line and the shaded area represents trajectories covering up to the worst realization. In general, the average constraint satisfaction for each monitoring region is well over their respective specification. More importantly, all trajectories satisfy the problem requirements, as Theorem \ref{theo_main_intro} guarantees.

\begin{figure}[t!]
    \centering
    \subfigure[Dual variables. Instantaneous (solid) and average (dashed) values shown.]{
	\includegraphics[scale=1]{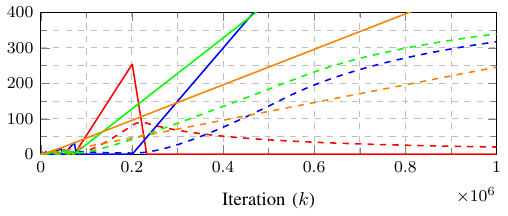} 			
     \label{fig:dual_variables_pd}
    }
    \subfigure[Constraint satisfaction. Requirements shown as dashed lines.]{
	\includegraphics[scale=1]{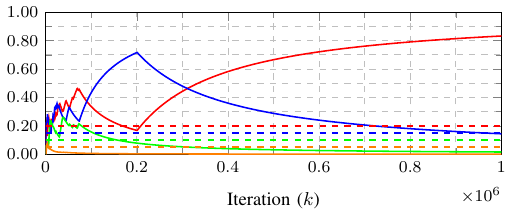} 				
     \label{fig:constraints_pd}
	 }	 
	\caption{Primal-dual algorithm. Dual variables and constraint satisfaction resulting from executing a primal-dual algorithm as described in Section \ref{sec_primal_dual}. A single realization is shown with colors matching the regions in Fig. \ref{fig:quiver}. The primal-dual method fails to solve the monitoring problem.}
    \label{fig:dual_constraints_pd}
\end{figure}

The success of the A-CRL algorithm is in contrast to the failure to solve the monitoring problem by using a vanilla primal-dual algorithm. We discussed previously, for the simple three-state MDP of Example \ref{ex_monitoring}, the lack of convergence of primal-dual methods in the CRL setting, resulting in oscillations and requiring continuous online execution (See Fig. \ref{fig:example_sim}). For the more complex case studied in this section, despite training the primal-dual method for the same number of iterations as the A-CRL algorithm ($1{,}000{,}000$ iterations), as shown in Fig. \ref{fig:dual_constraints_pd}, the primal-dual method is unable to generate trajectories satisfying the problem constraints. 

\section{Conclusions}

In this paper, we have studied CRL problems. The main motivation behind this work has been the observation that CRL and regularized RL problems are not equivalent, in the sense that the latter are unable to specify tasks that the former can. This concern also extends to dual approaches, precluding their guaranteed converge to optimal CRL policies. As illustrated by Example \ref{ex_monitoring}, a large class of complex behaviors, which we normally associate with autonomous agents, cannot be learned with current reinforcement learning methodologies. To overcome these limitations we have introduced the A-CRL algorithm. This proposed method, uses the dual variables as part of an augmented state space over which Lagrangian maximizing policies are learned. During execution, these policies are coupled with dual dynamics, producing trajectories which are guaranteed to be feasible and near-optimal.

\appendices  
  
\section{Proof of Theorem 1}\label{sec_unbiased}

In this section, we present the proofs of Theorem \ref{theo_main_intro}. We start with definitions that will be used throughout the appendix. We start by defining the following probability space $(\Omega,\calF,\mathbb{P})$, where the sample space is given by $\Omega = \calS\times\calA\times\mathbb{R}^m_+$, the event space $\calF$ corresponds the Borel $\sigma$-algebra, and $\mathbb{P} : \calF \to [0,1]$ is the probability function. Further, we define a filtration $\left\{\calF_k\right\}_{k\geq 0}$ where  $\calF_k\subset \calF$ is an increasing sequence of $\sigma$-algebras, i.e., if $k_2>k_1$ then $\calF_{k_1}\subset \calF_{k_2}$, with $\calF_0=\left\{\emptyset, \Omega\right\}$ and $\calF_{\infty} = \calF$. In particular $\calF_k$ are such that $\lambda_k$ is $\calF_k$-measurable. 

Preliminaries aside, let us denote by $\{p(\lambda_{k} | \lambda_0)\}_{ k\geq 0}$ the sequence of probability measures defined by the transitions of the Lagrange multipliers according to \eqref{eqn_stochastic_dual_descent_specified}. The bulk of the proof of Theorem \ref{theo_main_intro} can be reduced to the fact that the aforementioned sequence is tight. We formally define the notion of tightness next. 
\begin{definition}\label{def_tight}
We say that a sequence of probability measures $\{p(\lambda_{k} | \lambda_0)\}$ is tight, if for any $\delta>0$, there exists a compact set $\calK_\delta$ such that for every $k\geq 0$ 
\begin{equation}
    \mathbb{P}\left(\lambda_k \in \calK_\delta\right) >1-\delta.
\end{equation}
\end{definition}
Since the dual multipliers follow dynamics that resemble dual descent on a convex function, it should not come as a surprise that they belong to a compact set with high probability; and hence, why the sequence of measures is tight. A formal analysis of these arguments will be presented next.

In Section \ref{sec_feasibility} we establish that if the sequence of probability measures $\{p(\lambda_{k} | \lambda_0)\}$ is tight, then the feasibility claim of Theorem \ref{theo_main_intro} holds. Then, in Section \ref{sec_optimality}, having established the feasibility of the sequence, we focus on proving the optimality claim of Theorem \ref{theo_main_intro}. Then, for the proof to be complete, we devote Section \ref{sec_tight1} to show that indeed the sequence of probability measures is tight. 
 
\subsection{Feasibility Guarantees}\label{sec_feasibility}

In this section we establish the feasibility guarantees. In particular, we establish that if the sequence of measures $\{p(\lambda_{k} | \lambda_0)\}$ is tight, then expression \eqref{eqn_theo_feasibility_informal} holds. We state this result formally in the following proposition.
\begin{proposition}\label{prop_feasibility}
Consider the sequence of probabilities measures $\{p(\lambda_{k} | \lambda_0)\}$ defined by the dual update \eqref{eqn_stochastic_dual_descent_specified}. If the sequence $\{p(\lambda_{k} | \lambda_0)\}$ is tight, the state-action trajectories $(s_t,a_t)$ are feasible with probability one, i.e.,
\begin{align}
    \lim_{T\to\infty} \frac{1}{T}\sum_{t=0}^T r_i(s_t,a_t) \geq c_i, 
    \quad \text{~a.s.}
\end{align}
\end{proposition}
\begin{proof}
Let us argue by contradiction. Assume that there exist $\beta\in (0,1]$ and $\varepsilon>0$ such that for some $i =1,\ldots,m$ we have that 
  	\begin{equation}\label{eqn_contradiction}
  	    \mathbb{P}\left(
  	\liminf_{T\to\infty} \frac{1}{T}\sum_{t=0}^{T-1}  r_i(s_t,a_t) \leq c_i-\varepsilon \right) = \beta.
  	\end{equation}
  	  We set our focus now in lower bounding the norm of the $k+1$ iterate of the dual multiplier $\lambda_{i,k+1}$ whose evolution is defined by the dual descent step \eqref{eqn_stochastic_dual_descent_specified}. Notice that projection onto the non-negative orthant is such that for any $k\geq 0$ and for any $i=1,\ldots,m$ one has that 
  \begin{align}
    \lambda_{i,k+1} &= \left[\lambda_{i,k}-\eta_\lambda \frac{1}{T_0}\sum_{t= kT_0}^{(k+1)T_0-1}(r_i(s_t,a_t)-c_i)\right]_+ \nonumber\\
    &\geq \lambda_{i,k}-\eta_\lambda \frac{1}{T_0}\sum_{t= kT_0}^{(k+1)T_0-1}(r_i(s_t,a_t)-c_i).
    \end{align}
  Applying the previous inequality recursively it follows that 
  \begin{equation}
    \lambda_{i,k+1} \geq \lambda_{i,0} -\frac{\eta_\lambda}{T_0} \sum_{t=0}^{(k+1)T_0-1} \left(r_i(s_t,a_t) -c_i\right). 
  \end{equation}
  By the definition of $\left\|\lambda\right\|_1$ one has that $\left\|\lambda\right\|_1\geq \lambda_i$ for all $i=1,\ldots,m$ and for any $\lambda \in\mathbb{R}^m_+$. Hence, it holds that
  \begin{align}
    &\limsup_{k\to\infty} \left\|\lambda_{k+1}\right\|_1 \geq \limsup_{k\to\infty} \lambda_{i,k+1}  \nonumber\\
    &\geq  \lambda_{i,0} + \limsup_{k\to\infty} \left[ -\frac{\eta_\lambda}{T_0} \sum_{t=0}^{(k+1)T_0-1} \left(r_i(s_t,a_t) -c_i\right) \right].
  \end{align}
  Using~\eqref{eqn_contradiction}, we obtain that with probability $\beta$,
  \begin{align}
   &\limsup_{k\to\infty} \left[ -\frac{\eta_\lambda}{T_0} \sum_{t=0}^{(k+1)T_0-1} \left(r_i(s_t,a_t) -c_i\right) \right]  \nonumber\\
   &=-\eta_\lambda \liminf_{k\to\infty} (k+1) \frac{1}{(k+1)T_0} \sum_{t=0}^{(k+1)T_0-1} \left(r_i(s_t,a_t) -c_i\right)
   \nonumber\\
   &\geq \liminf_{k\to\infty} \eta_\lambda (k+1) \epsilon = \infty.
  \end{align}
Hence, for any compact set $\calK$ there exists a constant $J_\calK\geq 0$ and a subsequence $\left\{k_j\right\}$ such that $\mathbb{P}\left(\lambda_{k_j}\in\calK\right) = 1-\beta$ for $j>J_\calK$. This implies that for $\delta\in(0,\beta]$, there is no compact set $\calK$ for which $\mathbb{P}(\lambda_k\in \calK)> 1-\delta$, which contradicts the fact that the sequence of probabilities $\{p(\lambda_k | \lambda_0)\}$ is tight. This completes the proof of the proposition.
\end{proof}
 
\subsection{Optimality Guarantees}\label{sec_optimality}

Having established the feasibility part of the result, we now focus on establishing its optimality. To do so, we first require the following lemma stating the ergodic complementary slackness of the Lagrange multipliers and the estimates of the subgradients \eqref{eqn_stochastic_dual_descent_specified}.
\begin{lemma}\label{lemma_complementary_slackness}
Under the assumptions of Theorem \ref{theo_main_intro}, ergodic complementary slackness holds, i.e.,
  \begin{equation}
    \limsup_{K\to\infty} \frac{1}{KT_0} \sum_{k=0}^{K-1} \mathbb{E}\left[\lambda_k^\top\sum_{t=kT_0}^{(k+1)T_0-1}\left(r(s_t,a_t)-c\right)\right] \leq \eta_\lambda \frac{B^2}{2},
  \end{equation}
where as defined in Theorem \ref{theo_main_intro}, there exists $B$ such that $|r_i(s,a) - c_i| \leq B$ for all states $s$, actions $a$ and constraints $i$.
\end{lemma}

\begin{proof}
For notation simplicity, let us define $g_k \triangleq \sum_{t=kT_0}^{(k+1)T_0-1}\left(r(s_t,a_t)-c\right)/T_0$. Then, write the square of the dual iterates' norm at time $k+1$. Using the expression for the dual update \eqref{eqn_stochastic_dual_descent_specified} and the non-expansiveness property of the projection, one has that
  \begin{equation}
    \left\|\lambda_{k+1}\right\|^2 \leq \left\|\lambda_k-\eta_\lambda g_k \right\|^2 \leq \left\|\lambda_k\right\|^2 + \eta_\lambda^2 \left\|g_k\right\|^2-2\eta_\lambda\lambda_k^\top g_k,
  \end{equation}
  where the last inequality follows from expanding the squares. Then, using the fact that the difference $|r_i(s,a) - c_i|$ is bounded by $B$ in the previous expression and taking the conditional expectation with respect to the sigma algebra $\calF_k$ one has that
  \begin{equation}
\mathbb{E}\left[    \left\|\lambda_{k+1}\right\|^2 \mid \calF_k\right] \leq \left\|\lambda_k\right\|^2+\eta_\lambda^2B^2-2\eta_\lambda\mathbb{E}\left[\lambda_k^\top g_k\mid \calF_k\right].
\end{equation}
Applying the previous expression recursively yields
\begin{align}
\mathbb{E}\left[    \left\|\lambda_{k+1}\right\|^2 \mid \calF_0\right] &\leq \left\|\lambda_0\right\|^2+(k+1)\eta_\lambda^2B^2 \nonumber\\
&-2\eta_\lambda\sum_{\ell=0}^k\mathbb{E}\left[\lambda_\ell^\top g_\ell\mid\calF_0\right].
  \end{align}
For any random variable $X$ that is $\calF_0$ measurable we have that  $\mathbb{E}\left[X\mid\calF_0 \right]=\mathbb{E}\left[X \right]$. Using this observation, the fact that $\left\|\lambda_{k}\right\|\geq 0$ for all $k\geq 0$ and rearranging the terms in the previous expression the sum can be upper bounded by
\begin{equation}
\frac{1}{k+1}\sum_{\ell=0}^k\mathbb{E}\left[\lambda_l^\top g_l\right] \leq\frac{ \left\|\lambda_0\right\|^2}{2\eta_\lambda (k+1)}+\eta_\lambda\frac{B^2}{2}.
  \end{equation}
The proof is completed by taking the limit superior in both sides of the inequality.
\end{proof}

We are now in good condition to establish the optimality claim of Theorem \ref{theo_main_intro}. In addition to the hypothesis of Theorem \ref{theo_main_intro}, we also assume the tightness of the sequence of probabilities $\{p(\lambda_{k} | \lambda_0)\}$ induced by the update \eqref{eqn_stochastic_dual_descent_specified}. The fact that this sequence is tight under the hypothesis of Theorem \ref{theo_main_intro} is established later in Section \ref{sec_tight1}. 
\begin{proposition}\label{prop_optimality}
Let the hypothesis of Theorem \ref{theo_main_intro} hold. In addition, assume that sequence  $\{p(\lambda_{k} | \lambda_0)\}$ defined by the dual update \eqref{eqn_stochastic_dual_descent_specified} is tight. Then, the state-action trajectories $(s_t,a_t)$ are near-optimal in the following sense
\begin{align}
  \mathbb{E}\left[\lim_{T\to\infty} \frac{1}{T} \sum_{t=0}^T r_0(s_t,a_t)\right] \geq P^\star-\eta_\lambda \frac{B^2}{2}.
\end{align}
\end{proposition}
\begin{proof}

Let $\{\lambda_k\}$ be the sequence generated by the dual descent step in \eqref{eqn_stochastic_dual_descent_specified}. Define, the average Lagrange multiplier at time $K$ as $ \bar{\lambda}_K \triangleq \frac{1}{K}\sum_{k=0}^{K-1}\lambda_k$. Since the dual function is always an upper bound on the value of the primal and it is convex \cite{boyd2004convex} we have the following chain of inequalities 
  \begin{equation}\label{eqn_aux_theo_opt}
P^\star \leq  d(\bar{\lambda}_K) \leq \frac{1}{K} \sum_{k=0}^{K-1} d(\lambda_k).
  \end{equation}
  Let $\Pi(\lambda_k)$ be the set defined in \eqref{eqn_lagrangian_maximizers}. Then, for any $\pi(\lambda_k)\in \Pi(\lambda_k)$ we can rewrite the dual function as 
 $ d(\lambda_k) = V_0(\pi(\lambda_k))+\sum_{i=1}^m\lambda_{i,k}\left(V_i(\pi(\lambda_k))-c_i\right)$. Using the fact that the estimate \eqref{eqn_unbiased} is unbiased it follows that
 \begin{align}
   & \frac{1}{KT_0}\mathbb{E}\left[ \sum_{t=0}^{(K-1)T_0-1} r_0(s_t,a_t)\right]  \nonumber\\
   & \geq P^\star - \frac{1}{K} \sum_{k=0}^{K-1}\mathbb{E}\left[\lambda_k^\top\frac{1}{T_0}\sum_{t=kT_0}^{(k+1)T_0-1} \left(r(s_t,a_t)-c\right)\right]
\end{align}
Using the ergodic complementary slackness result derived from Lemma \ref{lemma_complementary_slackness}, the limit inferior of the average expected value can be lower bounded by 
\begin{equation}
  \liminf_{T\to\infty} \frac{1}{T} \mathbb{E} \left[ \sum_{t=0}^T r_0(s_t,a_t) \right]\geq P^\star-\eta_\lambda \frac{B^2}{2}.
\end{equation}
Since the formulation of the problem \eqref{P:constrainedRL} assumes the existence of the limit in the left hand side of the previous expression, the proof is complete. 
\end{proof}

\subsection{Tightness of the Sequence the Probabilities}\label{sec_tight1}
In this section we complete the proof of Theorem \ref{theo_main_intro} by establishing that the sequence of measures induced by \eqref{eqn_stochastic_dual_descent_specified} is tight. We start the development of this section by establishing the compactness of the following set
\begin{equation}\label{eqn_bounded_lambda}
    \calD \triangleq \left\{\lambda \in \mathbb{R}^m_+\mid d(\lambda)-P^\star\leq \eta_\lambda B^2/2\right\},
\end{equation}
where $P^\star$ is the value defined in \eqref{P:constrainedRL}, $\eta_\lambda$ the stepsize in \eqref{eqn_stochastic_dual_descent_specified} and $B$ the constant defined in Theorem \ref{theo_main_intro}. By virtue of Theorem \ref{theo_no_gap} it follows that the set $\calD$ is not empty. 
\begin{lemma}\label{lemma_bounded_lambda}
Let $\pi^\dagger$ be a strictly feasible policy. This is, there exists $C>0$ such that for every $i=1,\ldots,m$, $V_i(\pi^\dagger)-c_i\geq C$. Then, the set $\calD$ defined in \eqref{eqn_bounded_lambda} is contained in the following ball
\begin{equation}
    \left\|\lambda\right\|_1 \leq \frac{P^\star -V_0(\pi^\dagger)+\eta_\lambda B^2/2}{C},
\end{equation}
 where $\left\|\lambda \right\|_1=\sum_{i=1}^m \lambda_i$.

\end{lemma}
\begin{proof}
Let, $\pi^\dagger$ be a strictly feasible policy such that $V(\pi^\dagger)-c\geq C>0$. The existence of such policy is guaranteed by the hypothesis of Theorem \ref{theo_main_intro}. Then, by definition of the dual function \eqref{eqn_crl_dual_function}, one can lower bound the difference $d(\lambda)-P^\star$ as %
\begin{equation}
 d(\lambda)-P^\star \geq  V_0(\pi^\dagger)+\sum_{i=1}^m\lambda_i\left( V_i(\pi^\dagger)-c_i\right) - P^\star.    
\end{equation}
Using the fact that $\pi^\dagger$ is strictly feasible and that $\lambda\in\mathbb{R}^m_+$, we can further lower bound the difference $d(\lambda)-P^\star$ as follows 
\begin{equation}
 d(\lambda)-P^\star \geq  V_0(\pi^\dagger)+C\left\|\lambda\right\|_1 - P^\star.   
\end{equation}
Thus, for every element in $\lambda\in \calD$, it follows that \begin{equation}
    \left\|\lambda\right\|_1 \leq \frac{P^\star -V_0(\pi^\dagger)+\eta_\lambda B^2/2}{C}.
\end{equation}
This completes the proof of the result.
\end{proof}

\begin{lemma}\label{lemma_tight}
Under the hypothesis of Theorem \ref{theo_main_intro}, the sequence of probabilities $\{p(\lambda_{k} | \lambda_0)\}$ is tight according to Definition \ref{def_tight}.
\end{lemma}
\begin{proof}

Let $\calD$ be the set defined in \eqref{eqn_bounded_lambda} and $\calD^c$ its complement. By virtue of Lemma \ref{lemma_bounded_lambda} it follows that $\calD$ is contained in a compact set. Hence, to prove the result it suffices to show that the sequence of probabilities $\{p(\lambda_k | \lambda_0)\}$ is tight in the complement of $\calD$. 

In addition, denote by $\calB_{\eta_\lambda B}$ a ball centered at $0$ and radius $\eta_\lambda B$ and by $\oplus$ the Minkowski addition. Notice that if $\lambda_k\in\calD$, because the estimate of the gradient in \eqref{eqn_stochastic_dual_descent_specified} is bounded by $B$, then we have that $\lambda_{k+1}\in\calD\oplus\calB_{\eta_\lambda B}$. This means that the first iterate that is in $\calD^c$ is such that is has bounded norm with probability one.  

Hence, without loss of generality assume that the sequence $\left\{\lambda_k\right\}$ is such that $\lambda_0\in\calD^c$. Let us define the following stopping time $K_0 = \min_{k\geq0} \left\{\lambda_k\in\calD\right\}$. And define the sequence $\lambda_{k \wedge K_0}$, where $\wedge$ denotes the minimum between $k$ and $K_0$. Note that the previous discussion allows us to consider only indices $k\leq K_0$. Indeed, at $k=K_0$ the iterate belongs to the compact set $\calD$ and one can re-define an index $j = k-K_0-K_0^\prime$ where $K_0^\prime$ is the number of iterations needed for $\lambda_k \notin \calD$.   

We denote by $\left\| \lambda- \lambda^\star\right\|^2$ the distance from $\lambda$ to the optimal set of multipliers in  \eqref{eqn_crl_dual_problem}. Then, using the non expansiveness of the norm and substituting  $\lambda_{(k+1)\wedge K_0}$ by its dynamics \eqref{eqn_stochastic_dual_descent_specified} it follows that 
\begin{align}
   & \left\| \lambda_{(k+1)\wedge K_0}- \lambda^\star\right\|^2 \leq \left\| \lambda_{k\wedge K_0}-\eta_\lambda g_k-\lambda^\star\right\|^2 \nonumber\\
    &= \left\|\lambda_{k\wedge K_0}-\lambda^\star\right\|^2- 2\eta_\lambda\left(\lambda_{k\wedge K_0}-\lambda^\star\right)^\top g_k+\eta_\lambda^2\left\|g_k\right\|^2,
\end{align}
where we have written for notation simplicity $g_k \triangleq \sum_{t=kT_0}^{(k+1)T_0-1}\left(r(s_t,a_t)-c\right)/T_0$. 
Next, consider the conditional expectation of $\left\| \lambda_{(k+1)\wedge K_0}- \lambda^\star\right\|^2$ with respect to the sigma algebra $\calF_k$.  %
\begin{align}
   & \mathbb{E}\left[\left\| \lambda_{(k+1)\wedge K_0}- \lambda^\star\right\|^2\big| \calF_k\right] \leq  \left\|\lambda_{k\wedge K_0}-\lambda^\star\right\|^2 \nonumber\\
   &- 2\eta_\lambda\left(\lambda_{k\wedge K_0}-\lambda^\star\right)^\top\mathbb{E}\left[g_k\big|\calF_k\right]+\eta_\lambda^2\mathbb{E}\left[\left\|g_k\right\|^2\right],
\end{align}
where in the previous expression we have used the fact that $\lambda_{k\wedge K_0}$ is $\calF_k$ measurable. Observe that the rewards are bounded (cf., Theorem \ref{theo_main_intro}) it follows that the second moment of the estimate is bounded by $B^2$. Hence the previous expression reduces to 
\begin{align}
   & \mathbb{E}\left[\left\| \lambda_{(k+1)\wedge K_0}- \lambda^\star\right\|^2\big| \calF_k\right] \leq  \left\|\lambda_{k\wedge K_0}-\lambda^\star\right\|^2 \nonumber\\
   &- 2\eta_\lambda\left(\lambda_{k\wedge K_0}-\lambda^\star\right)^\top \mathbb{E}\left[g_k\big|\calF_k\right]+\eta_\lambda^2B^2.
\end{align}
In addition use the fact that $g_k$ is an unbiased estimator of $V(\pi(\lambda_{k\wedge K_0}))-c$, which by Danskin's Theorem \cite{danskin2012theory} is in the subgradient of the dual function. It follows from the convexity of the dual function \cite{boyd2004convex} that  $\left(\lambda_{k\wedge K_0}-\lambda^\star\right)^\top\mathbb{E}\left[g_k\big|\calF_k\right] \leq d(\lambda_{k\wedge K_0}) - d(\lambda^\star)$. Substituting in the previous expression it yields
\begin{align}
   & \mathbb{E}\left[\left\| \lambda_{(k+1)\wedge K_0}- \lambda^\star\right\|^2\big| \calF_k\right] \leq  \left\|\lambda_{k\wedge K_0}-\lambda^\star\right\|^2 \nonumber\\
   &- 2\eta_\lambda\left(d(\lambda_{k\wedge K_0})-d(\lambda^\star)\right)+\eta_\lambda^2B^2.
\end{align}
Using the fact $d(\lambda^\star)\geq P^\star$ and the definition of the set $\calD$  \eqref{eqn_bounded_lambda}, for any  $\lambda_{k\wedge K_0}\notin\calD$ the previous expression reduces to
\begin{equation}\label{eqn_tight_last}
    \mathbb{E}\left[\left\| \lambda_{(k+1)\wedge K_0}- \lambda^\star\right\|^2\big| \calF_k\right] \leq  \left\|\lambda_{k\wedge K_0}-\lambda^\star\right\|^2. 
\end{equation}
Applying the equation recursively and taking expectation it follows that $\mathbb{E}\left[\left\| \lambda_{(k+1)\wedge K_0}- \lambda^\star\right\|^2\right] \leq \left\|\lambda_0-\lambda^\star \right\|^2.$

To complete the proof, for any $\delta>0$ define the set $\calK_\delta = \left\{\lambda\in\mathbb{R}_+^m \mid \left\|\lambda-\lambda^\star\right\|^2\leq  \left\|\lambda_0-\lambda^\star\right\|^2/\delta\right\}$. Notice that this set is closed by definition. In addition, the existence of a strictly feasible policy (see (A1) in Theorem \ref{theo_main_intro}) guarantees that the set of optimal Lagrange multipliers is bounded \cite{bertsekas1999nonlinear}. Thus $K_\delta$ is compact. Hence, it follows from the Markov inequality that 
\begin{align}
\mathbb{P} \left(\lambda_{k\wedge K_0}\in \calK_\delta\right) &=\mathbb{P} \left(\left\|\lambda_{k\wedge K_0}-\lambda^\star\right\|^2 \leq\frac{\left\|\lambda_{0}-\lambda^\star\right\|^2}{\delta} \right)  \nonumber\\
&> 1-\delta\frac{\mathbb{E}\left[\left\|\lambda_{k\wedge K_0}-\lambda^\star\right\|^2\right]}{\left\|\lambda_{0}-\lambda^\star\right\|^2} \geq 1-\delta,
\end{align}

where the last inequality follows from \eqref{eqn_tight_last}. Hence completing the proof of the result. 
\end{proof}

\section{Results for a Realizable Estimator}\label{app:sec_realizable}

In this section we present a result akin to Theorem \ref{theo_main_intro}, but in which we lift the assumption of the estimates in \eqref{eqn_stochastic_dual_descent_specified} being unbiased. Let us remark, that is a customary assumption in reinforcement learning, specially for continuing tasks, to assume that one can sample state-action pairs that are distributed according to a stationary distribution of the MDP that is independent of the initial state~\cite{sutton2018reinforcement}. Such an assumption is equivalent to (A3) in Theorem \ref{theo_main_intro}. 

In practice, it is possible that the previous hypothesis might not hold. Formally, this prevents us from having access to an unbiased estimate of the subgradients of the value functions. To be able to compute these unbiased estimates one needs to rollout the system for an infinite number of steps before modifying the policy. Clearly, yielding an unrealizable algorithm. Indeed, as long as as one does not modify the policy, the Markov decision process defines the following occupancy measure
\begin{equation}
\rho_\pi(s,a) \triangleq \lim_{T\to\infty} \frac{1}{T}\sum_{t=0}^{T-1} p(s_t=s,a_t=a | \pi ).
\end{equation}
As an approximation to sampling from such a distribution, we use a rollout defined until a time horizon $T_0$ and use as estimate of the gradient the expression defined in \eqref{eqn_stochastic_dual_descent_specified}. Since, the horizon $T_0$ is finite, there is no guarantee that the described sampling strategy provides an unbiased estimator. However, for long enough $T_0$ the bias in the gradient can be assumed to be made arbitrarily small. Notice that this is a direct consequence of assuming the existence of a stationary distribution for each policy $\pi$. In particular, we assume that sampling from the trajectory under a given policy $\pi$ for long enough time is an arbitrarily good approximation to sampling from the stationary distribution as long as $T_0$ is long enough. We formalize this assumption next.
\begin{assumption}\label{assumption_ams}
For any $\epsilon>0$ and $\beta>0$, there exists $T_0 > 0$ such that 
\begin{align}
    \left\|\mathbb{E}_{s,a_t\sim \pi}\left[\frac{1}{T_0}\sum_{t=kT_0}^{(k+1)T_0-1}r(s_t,a_t)\big| \calF_k\right]  - \right. \hspace{12ex}\nonumber\\
     \left.\mathbb{E}_{(s_t,a_t)\sim\rho_\pi} \left[\frac{1}{T_0}\sum_{t=kT_0}^{(k+1)T_0-1}r(s_t,a_t)\big| \calF_k\right] \right\| \leq \epsilon
\end{align}
and 
\begin{align}
    \left|\mathbb{E}_{s,a_t\sim \pi}\left[\frac{1}{T_0}\sum_{t=kT_0}^{(k+1)T_0-1}r_{\lambda_k}(s_t,a_t)\big| \calF_k\right]  - \right. \hspace{10ex} \nonumber\\
    \left. \mathbb{E}_{(s_t,a_t)\sim\rho_\pi} \left[\frac{1}{T_0}\sum_{t=kT_0}^{(k+1)T_0-1}r_{\lambda_k}(s_t,a_t)\big| \calF_k\right] \right| \leq \beta,
\end{align}
where $r_\lambda(s_t,a_t)$ is the reward as defined in \eqref{eqn_lagrangian_reward}.

\end{assumption}
This assumption is milder than the one stated in Theorem \ref{theo_main_intro}. Indeed, (A3) from Theorem \ref{theo_main_intro} is recovered from Assumption \ref{assumption_ams} with $\epsilon=\beta =0$. To relax this assumption in the sampling time, we require the following technical assumption regarding the dual function. 
\begin{assumption}\label{assumption_extra}
Let us denote by $\calB_\epsilon \subset \mathbb{R}^m$ the ball with center at zero and radius $\epsilon$. Further, define the following set 
\begin{align}
\Lambda_{\eta,\epsilon} = \bigl\{ \lambda\in\mathbb{R}^m_+ \mid &  \mbox{ for all} \, v\in \calB_\epsilon, g\in \partial d(\lambda), \nonumber\\
& \left( g+v\right)^\top(\lambda-\lambda^\star) \geq \eta B^2/2  \bigr\}. 
\end{align}
Then, there exists $\eta,\epsilon, B_\lambda>0$ such that $\Lambda_{\eta,\epsilon} \neq \emptyset$ and for every $\lambda \in \Lambda_{\eta,\epsilon}^c$, we have $\left\|\lambda-\lambda^\star\right\|\leq B_\lambda$.
\end{assumption}
The interpretation of this assumption is that if the term $\lambda-\lambda^\star$ is large, then the increment in the function $g^\top\left(\lambda-\lambda^\star\right)$ has to be larger than linear. Recall that the dual function \eqref{eqn_crl_dual_function} is always convex regardless of the structure of the primal problem \cite{boyd2004convex}. In this sense, Assumption \ref{assumption_extra} is a stronger assumption than convexity but milder than strong convexity. 

In summary, Assumptions \ref{assumption_ams} and \ref{assumption_extra} replace the hypothesis (A3) in Theorem \ref{theo_main_intro}. With this modification, the claims in terms of feasibility remain the same. We formalize this result next. 
\begin{theorem}\label{theo_second_version}
An agent chooses actions according to $a_t \sim \pi(\lambda_k)$, where $\pi(\lam_k)\in\Pi(\lam_k)$ with $\lambda_k$ updates following \eqref{eqn_stochastic_dual_descent_specified}. Assume that
\begin{description}
\item[(A1)] There exists a strictly feasible policy $\pi^\dagger$ such that for some $C>0$ and all $i$ constraints $V_i(\pi^\dagger)-c_i \geq C$.
\item[(A2)] There exists $B$ such that $|r_i(s,a) - c_i| \leq B$ for all states $s$, actions $a$ and constraints $i$.
\item[(A3)] Assumptions \ref{assumption_ams} and \ref{assumption_extra} hold. 
\end{description}
Then, the state-action sequences $(s_t,a_t)$ are feasible with probability one,
\begin{align}\label{eqn_theo_feasibility_informal2}  
    \liminf_{T\to\infty} \frac{1}{T}\sum_{t=0}^T r_i(s_t,a_t) \geq c_i, 
    \quad \text{~a.s.}
\end{align}
And they are near-optimal in the following sense
\begin{align}\label{eqn_theo_optimality_informal2}
  \lim_{T\to\infty} \mathbb{E}\left[\frac{1}{T} \sum_{t=0}^T r_0(s_t,a_t) \right]\geq P^\star-\eta_\lambda \frac{B^2}{2} - \beta,
\end{align}
where $\beta>0$ is the constant defined in Assumption \ref{assumption_ams}.
\end{theorem}
We devote the rest of this section to prove the previous Theorem. We start by leveraging some results from Appendix \ref{sec_unbiased}. By Proposition \ref{prop_feasibility}, to guarantee feasibility, it suffices to show the sequence of probability measures induced by the update \eqref{eqn_stochastic_dual_descent_specified} is tight under the new hypothesis of Theorem \ref{theo_second_version}. The following lemma formalizes this claim. 

\begin{lemma}\label{lemma_tight2}
Under the hypothesis of Theorem \ref{theo_second_version} the sequence of probability measures $\{p(\lambda_k | \lambda_0)\}$ is tight. 
\end{lemma}
\begin{proof}
Let $\eta,\epsilon,B_\lambda>0$ be the constants from Assumption \ref{assumption_extra}. Further, by virtue of Assumption \ref{assumption_extra}, the set $\Lambda_{\eta,\epsilon}^c$ is compact. Hence, in a similar manner to the proof in Lemma \ref{lemma_tight} it suffices to show that for any $\lambda_k\in\Lambda_{\eta,\epsilon}$ the following inequality holds, $\mathbb{E}\left[\left\|\lambda_k-\lambda^\star\right\|\right]\leq \mathbb{E}\left[\left\|\lambda_0-\lambda^\star\right\|\right]$. Let us consider the expectation of  $ \left\|\lambda_{k+1}-\lambda^\star\right\|$ conditioned on $\calF_k$
\begin{align}
\mathbb{E}\biggl[\bigl\| \lambda_{k+1} & -\lambda^\star\bigr\|^2  \mid \calF_k\biggr] \leq \mathbb{E}\biggl[\left\|\lambda_{k}-\lambda^\star\right\|^2 \nonumber\\
&- 2\eta \left(\lambda_k-\lambda^\star\right)^\top g_k +\eta^2 \left\|g_k\right\|^2  \mid \calF_k\biggr], 
\end{align}
where in the previous expression, we have defined for convenience $g_k \triangleq \sum_{t=kT_0}^{(k+1)T_0-1}\left(r(s_t,a_t)-c\right)/T_0$ and have used the non-expansive property of the projection. Recall that by virtue of (A2) in Theorem \ref{theo_second_version} the rewards are bounded. Hence, $\left\|g_k\right\|\leq B$. Using the linearity of the expectation and the fact that $\lambda_k$ is measurable with respect to $\calF_k$ it follows that 
\begin{align}
\mathbb{E}\biggl[\bigl\|\lambda_{k+1} &- \lambda^\star\bigr\|^2  \mid  \calF_k\bigg]  \leq \left\|\lambda_k-\lambda^\star\right\|^2 \nonumber\\
&-2\eta\left(\lambda_k-\lambda^\star\right)^\top\mathbb{E}\left[g_k\mid \calF_k\right] + \eta^2B^2.  
\end{align}
We claim that there exists $T_0$ such that $\mathbb{E}\left[g_k\mid \calF_k\right] \in \partial d(\lambda_k) \oplus \calB_\epsilon$. We will prove this claim at the end of the proof. By virtue of Assumption \ref{assumption_extra}, for any $\lambda_k \in \Lambda_{\eta,\epsilon}$ it follows that 
\begin{equation}
-2\eta\left(\lambda_k-\lambda^\star\right)^\top\mathbb{E}\left[g_k\mid \calF_k\right] +\eta^2B^2 < 0.
\end{equation}
 Hence, $\left\|\lambda_k-\lambda^\star \right\|^2$ is a supermartingale. Thus proving the tightness of the sequence of probabilities. We are left to show that  $\mathbb{E}\left[g_k\mid \calF_k\right] \in \partial d(\lambda_k) \oplus \calB_\epsilon$. We do so next. By virtue of Danskin's Theorem \cite{danskin2012theory,bertsekas1999nonlinear} and the definition of the Lagrangian \eqref{eqn_lagrangian} it follows that 
\begin{equation}
\lim_{T\to\infty} \frac{1}{T-kT_0} \mathbb{E}_{a_t\sim\pi(\lambda_k)}\left[ \sum_{t= kT_0}^{T-1}(r(s_t,a_t)-c)\mid \calF_k\right] \in \partial d(\lambda_k).
\end{equation}
Notice that the left hand side of the previous expression, corresponds to sampling from the stationary distribution $\rho_\pi(\lambda_k)$, thus it follows that 
\begin{align}
&\mathbb{E}_{(s_t,a_t)\sim \rho_\pi}\left[\frac{1}{T_0}\sum_{t=kT_0}^{(k+1)T_0-1} r(s_t,a_t)-c\mid \calF_k\right]  = \\ 
&\lim_{T\to\infty} \frac{1}{T-kT_0} \mathbb{E}_{a_t\sim\pi(\lambda_k)}\left[ \sum_{t= kT_0}^{T-1}(r(s_t,a_t)-c)\mid \calF_k\right] \in \partial d(\lambda_k) \nonumber
\end{align}
where the equality follows from the stationarity of $\rho_\pi$. Hence, we have that, 
\begin{equation}\label{eqn_lemma4_aux_1}
    \mathbb{E}_{(s_t,a_t)\sim \rho_\pi}\left[\frac{1}{T_0}\sum_{t=kT_0}^{(k+1)T_0-1} \hspace{-2ex}r(s_t,a_t)-c\mid \calF_k \right] \in \partial d(\lambda_k).
\end{equation}
By virtue of Assumption \ref{assumption_ams} there exists $T_0$ such that %
\begin{equation}\label{eqn_lemma4_aux_2}
    \left\|\mathbb{E}\left[g_k\mid \calF_k\right] -    \mathbb{E}_{(s_t,a_t)\sim \rho_\pi}\left[\frac{1}{T_0}\sum_{t=kT_0}^{(k+1)T_0-1}\hspace{-2ex} r(s_t,a_t)-c\right]\right\| \leq \epsilon.
\end{equation}
Combining the expressions in \eqref{eqn_lemma4_aux_1} and \eqref{eqn_lemma4_aux_2}, we have that $g_k\in \partial d(\lambda_k)\oplus \calB_\epsilon$. This completes the proof of the result.
\end{proof}

Having established the feasibility of the update under the assumptions of Theorem \ref{theo_second_version}, we set our focus into proving the optimality claim. The proof of this is analogous to Proposition \ref{prop_optimality}. Since the ergodic complementary slackness result (cf. Lemma \ref{lemma_complementary_slackness}) is independent of Assumption (A3) it holds in this case as well. In the case of the unbiased estimate (Theorem \ref{theo_main_intro}) the result follows directly from these results using the fact that the sampled cumulative rewards 
\begin{equation}
    \frac{1}{T_0}\mathbb{E}_{a_t\sim \pi(\lambda_k)}\left[\sum_{t=kT_0}^{(k+1)T_0-1}\hspace{-2ex}r_0(s_t,a_t)+\lambda_k^\top r(s_t,a_t)\right]
\end{equation}
is the dual function evaluated at $\lambda_k$. In the context of Theorem \ref{theo_second_version} this does not hold exactly. However the difference can be bounded by virtue of Assumption \ref{assumption_ams} as
\begin{equation}
    \left| d(\lambda_k) - \frac{1}{T_0}\mathbb{E}_{a_t\sim \pi(\lambda_k)}\left[\sum_{t=kT_0}^{(k+1)T_0-1}\hspace{-2ex} r_0(s_t,a_t)+\lambda_k^\top r(s_t,a_t)\right] \right| \leq \beta.
\end{equation}
Taking this error into account and following analogous steps to the proof of Proposition \ref{prop_optimality} one completes the proof. 

\vspace{-1.25ex}

\section{Proof of Proposition \ref{prop_primal_recovery}}\label{app:prop_primal_recovery}

\begin{proof} As per the definition of the dual function \eqref{eqn_crl_dual_function} specialized to $\lambda=\lambda^\star$, we have that $d(\lambda^\star) = \max_{\pi} \calL(\pi,\lambda^\star)$. Recalling the definition in \eqref{eqn_lagrangian_maximizers_rl}, policies achieving this maximization are Lagrangian maximizing policies belonging to the set $\Pi(\lambda^\star)$. Further, evaluating the Lagrangian at any policy cannot exceed the maximum and we therefore have that $\max_{\pi} \calL(\pi,\lambda^\star)\geq \calL(\pi',\lambda^\star)$ for any policy $\pi'$. Specializing to an optimal policy $\pi^\star \in \Pi^\star$ yields 
\begin{align}\label{eqn_prop_primal_recovery_pf_10}
   d(\lambda^\star) = \max_{\pi} \calL(\pi,\lambda^\star) \geq  \calL(\pi^\star,\lambda^\star) .
\end{align}
Consider now the definition of the Lagrangian in \eqref{eqn_lagrangian} and observe that since $\pi^\star \in \Pi^\star$ is optimal, it must also be feasible. Hence, $\left(V_i(\pi^\star)-c_i\right)\geq0$ for all $i=1,\ldots,m$. Given that the Lagrange multipliers are also positive it must be that
\begin{align}\label{eqn_prop_primal_recovery_pf_20}
\calL(\pi^\star,\lambda^\star) = V_0(\pi^\star) + \sum_{i = 1}^m \lambda_i^\star \left(V_i(\pi^\star)-c_i\right)\geq V_0(\pi^\star) = P^\star ,
\end{align}
where we used the fact that for an optimal policy we must have $P^\star=V_0(\pi^\star)$. Combining \eqref{eqn_prop_primal_recovery_pf_10} and \eqref{eqn_prop_primal_recovery_pf_20} we conclude that
\begin{align}\label{eqn_prop_primal_recovery_pf_30}
   D^\star = d(\lambda^\star) = \max_{\pi} \calL(\pi,\lambda^\star)  \geq  \calL(\pi^\star,\lambda^\star) \geq  P^\star.
\end{align}
According to Theorem \ref{theo_no_gap}, there is no duality gap and it must therefore be that the inequalities in \eqref{eqn_prop_primal_recovery_pf_30} hold with equality. Thus $\max_{\pi} \calL(\pi,\lambda^\star) = \calL(\pi^\star,\lambda^\star)$ which implies 
that any optimal policy $\pi^\star \in \Pi^\star$ must be included in the set of Lagrangian maximizing polices for $\lambda=\lambda^\star$. Namely, $\Pi^\star \subseteq \Pi(\lam^\star)$. 
\end{proof}

\bibliographystyle{ieeetr}
\bibliography{bib}

\end{document}